\documentclass{article}

\usepackage{packages}
\PassOptionsToPackage{authoryear}{natbib}
\usepackage[final]{neurips_2020}

\usepackage{todonotes}

\usepackage{claimICML}
\usepackage{theoremsICML}
\usepackage{decoratedICML}
\usepackage{notationICML}

\makeatletter
\newcommand*{\addFileDependency}[1]{
  \typeout{(#1)}
  \@addtofilelist{#1}
  \IfFileExists{#1}{}{\typeout{No file #1.}}
}
\makeatother

\title{Regret Bounds without Lipschitz Continuity:\\
Online Learning with Relative-Lipschitz Losses}

\author{%
  Yihan Zhou\footnotemark[1]\\
  University of British Columbia
  \And
  Victor S.~Portella\footnotemark[1]\\
  University of British Columbia
  \And
  Mark Schmidt\\
  University of British Columbia\\
  CCAI Affiliate Chair (Amii)
  \And
  Nicholas J. A. Harvey\\
  University of British Columbia
}

\begin{document}

\maketitle

\renewcommand{\thefootnote}{\fnsymbol{footnote}}

\footnotetext[1]{Equal contributions.}

\begin{abstract}
In online convex optimization (OCO), Lipschitz continuity of the
functions
is commonly assumed in order to obtain
sublinear regret.
Moreover, many algorithms have only logarithmic regret when these functions are also strongly convex. Recently,
researchers from convex optimization proposed the notions of ``relative
Lipschitz continuity'' and ``relative strong convexity''. Both of the
notions are generalizations of their classical counterparts.
It has been shown that subgradient methods in the relative setting have performance analogous to their performance in the classical setting.

In this work, we consider OCO for relative Lipschitz and relative strongly convex functions. We extend the known regret bounds for classical OCO algorithms to the relative setting. Specifically, we show regret bounds for the follow
the regularized leader algorithms and a variant of online mirror
descent. Due to the generality of these methods, these results yield
regret bounds for a wide variety of OCO algorithms. Furthermore, we
further extend the results to algorithms with extra regularization such
as regularized dual averaging.
\end{abstract}

\renewcommand{\thefootnote}{\arabic{footnote}}

\section{Introduction}

In online convex optimization (OCO), at each of many rounds a player has
to pick a point from a convex set while an adversary chooses a convex
function that penalizes the player's choice. More precisely, in each
round $t \in \Naturals$, the player picks a point $x_t$ from a fixed
convex set $\mathcal{X} \subseteq \mathbb{R}^n$ and an adversary picks a
convex function $f_t$ depending on \(x_t\). At the end of the round, the
player suffers a loss of $f_t(x_t)$. Besides modeling a wide range of
online learning problems \citep{Shalev-Shwartz11a}, algorithms for OCO
are often used in batch optimization problems due to their low
computational cost per iteration. For example, the widely used stochastic gradient
descent (SGD) algorithm can be viewed as a special case of online gradient descent
\citep[Chapter~3]{Hazan16a} and AdaGrad
\citep{duchi2011adaptive} is a foundational adaptive gradient descent
method originally proposed in the OCO setting.
The performance measure usually used for OCO algorithms is the
\emph{regret}. It is the difference between the cost incurred to the
player and a comparison point \(z \in \Xcal \subseteq \Reals^n\)
(usually with minimum cumulative loss), that is to say,
\begin{equation*}
\Regret_T(z) \coloneqq \sum_{t=1}^T f_t(x_t)-  \sum_{t=1}^T f_t(z).
\end{equation*}
 Classical results show that if the cost functions are Lipschitz
 continuous, then there are algorithms which suffer at most
 \(O(\sqrt{T})\) regret in \(T\) rounds \citep{zinkevich2003online}.
 Additionally, if the cost functions are strongly convex, there are
 algorithms that suffer at most \(O(\log T)\) regret in \(T\) rounds
 \citep{hazan2007logarithmic}). However, not all loss functions that
 appear in applications, such as in inverse Poisson problems \citep{antonakopoulosonline} and support
 vector machines training \citep{lu2019relative}, satisfy these conditions on the entire
 feasible set.
 
 Recently, there has been a line of work investigating the performance
 of optimization methods beyond conventional assumptions
 \citep{bauschke2017descent, lu2018relatively, lu2019relative}.
 Intriguingly, much of this line of work proposes relaxed assumptions
 under which classical algorithms enjoy convergence rates similar to the
 ones from the classical setting.

 
In particular, \citet{lu2019relative} proposed the notion of relative
Lipschitz-continuity and showed how mirror descent (with properly chosen
regularizer/mirror map) converges at a rate of \(O(1/\sqrt{T})\) in
\(T\) iterations for non-smooth relative Lipschitz-continuous functions.
Furthermore, they show a \(O(1/T)\) convergence rate when the function
is also relatively strongly-convex (a notion proposed by
\citet{lu2018relatively}). Although the former result can be
translated to a \(O(\sqrt{T})\) regret bound for \emph{online mirror
descent} (OMD), the latter does not directly yield regret bounds in the
online setting. Moreover, \citet{orabona2018scale} showed that OMD is
not suitable when we do not know a priori the number of iterations since
it may suffer linear regret in this case. Finally, at present it is not
known how foundational OCO algorithms such as \emph{follow the
regularized leader} (FTRL)~\citep{Shalev-Shwartz11a, Hazan16a} and
\emph{regularized dual averaging}~\citep{xiao2010dual} (RDA) perform in
the relative setting.

\paragraph{Our results.} 
We analyze the performance of two general OCO algorithms: FTRL and
dual-stabilized OMD (DS-OMD, see~\citep{HuangHPF20a}).
 We give \(O(\sqrt{T})\) regret bounds in \(T\) rounds for relative
Lipschitz loss functions. Moreover, this is the first paper to show
\(O(\log T)\) regret if the loss functions are also relative
strongly-convex.\footnote{This can be seen as analogous to the known
logarithmic regret bounds when the loss functions are strongly
convex~\citep{hazan2007logarithmic}.} In addition, we are able to extend
these bounds for problems with composite loss functions, such as adding
the \(\ell_1\)-norm to induce sparsity. The generality of these
algorithms lead to regret bounds for a wide variety of OCO algorithms
(see~\citet{Shalev-Shwartz11a, Hazan16a} for some reductions).  We
demonstrate this flexibility by deriving convergence rates for
\emph{dual averaging}~\citet{nesterov2009primal} and \emph{regularized
dual averaging}~\citep{xiao2010dual}.



\subsection{Related Work}
\label{sec:related_work}

Analyses of gradient descent methods in the differentiable convex
setting usually require the objective function \(f\) to be Lipschitz
smooth, that is, the gradient of the objective function \(f\) is
Lipschitz continuous.
\citet{bauschke2017descent} proposed a generalized Lipschitz smoothness
condition, called \emph{relative Lipschitz smoothness}, using Bregman
divergences of a fixed reference function. They proposed a proximal
mirror descent method\footnote{They propose an algorithm in the general
case with composite functions, but when we set \(f \coloneqq 0\) in
their algorithm it boils down to classical mirror descent. In this case
the novelty comes from the convergence analysis at a \(O(1/T)\) rate
without the use of classical Lipschitz smoothness.} called NoLips with a
\(O(1/T)\) convergence rate for such functions.
\citet{van2017forward} independently developed similar ideas for analyzing the convergence of a Bregman proximal gradient method applied to convex composite functions in Banach spaces. \citet{bolte2018first} extended the framework of \citet{bauschke2017descent} to the non-convex setting.
Building upon this work, \citet{lu2018relatively} slightly relaxed the
definition of relative smoothness and gave simpler analyses for mirror
descent and dual averaging. \citet{Hanzely18a} propose and analyse
coordinate and stochastic gradient descent methods for relatively smooth
functions.  These ideas were later applied to non-convex problems by
\citet{mukkamala2019beyond}. More recently, \citet{gao2020randomized} analysed the coordinate descent method with composite Lipschitz smooth objectives. Unlike those prior works, in this paper we
focus on the online case with non-differentiable loss functions.


For non-differentiable convex optimization, Lipschitz continuity of the
objective function is usually needed to obtain a \(O(1/\sqrt{T})\)
convergence guarantee for classical methods. \citet{lu2019relative}
showed that this condition can be relaxed to what they called relative
Lipschitz continuity of the objective function. Under this latter
assumption, they gave \(O(1/\sqrt{T})\) convergence rates for
deterministic and stochastic mirror descent. In a similar vein,
\citet{grimmer2019convergence} showed how projected subgradient descent
enjoys a $O(1/\sqrt{T})$ convergence rate without Lipschitz continuity
given that one has some control on the norm of the subgradients. None of
these works considered online algorithms. Although the results from
\citet{lu2019relative} for mirror descent can be adapted to the online
setting, it is not clear how other foundational OCO algorithms such as
FTRL or RDA perform in this setting.

\citet{antonakopoulosonline} generalized the
Lipschitz continuity condition from the perspective of Riemannian
geometry. They proposed the notion of Riemann-Lipschitz continuity
(RLC) and analyzed how OCO algorithms perform in this setting. They
showed \(O(\sqrt{T})\) regret bounds for both FTRL and OMD with RLC cost
functions in both the online and stochastic settings. In Appendix~\ref{sec:RLC} we discuss in detail the relationship between RLC and relative Lipschitzness and how some of our regret bounds compare to those due to \citet{antonakopoulosonline}. In related work, \citet{maddison2018hamiltonian} relaxed the Lipschitz smoothness condition by proposing a new family of optimization methods motivated from physics, to be more specific, the conformal Hamiltonian dynamics.

Moreover, in the presence of both Lipschitz continuity and strong
convexity we can obtain \(O(1/T)\) convergence rates in classical convex
optimization~\citep[Section~3.4.1]{bubeck2015convex} and \(O(\log T)\)
regret in the online case~\citep{hazan2007logarithmic}.
By replacing the squared norm in the usual strong convexity inequality by a Bregman divergence of a fixed reference function yields the notion of \emph{relative strong convexity}. This idea dates back to the work of~\citet{HazanAK07a}. In recent work, \citet{lu2018relatively} showed algorithms with \(O(1/T)\) convergence rates in the offline setting when the objective function is both relative
Lipschitz continuous and relative strongly convex. Still, this latter
work does not obtain regret bounds for the online case. \citet{HazanAK07a} analyze the online case and show logarithmic regret bounds for online mirror descent when the cost functions are strongly convex relative to the mirror map. However, they assume (classical) strong convexity of the mirror map, which ultimately implies that the cost function need also be strongly convex.\footnote{More precisely, the regret bound in \citep[Theorem~1]{HazanAK07a} requires the cost functions \((g_t)_{t \in \Naturals}\) to be strongly convex relative to the mirror map \(f\). In turn, the result also requires \(f\) to be strongly convex (in the classical sense) with respect to a fixed norm \(\norm{\cdot}\). This implies that the cost functions \((g_t)_{t \in \Naturals}\) are strongly convex w.r.t.\ \(\norm{\cdot}\) as well.} To the best of
our knowledge, this is the first work studying conditions beyond strong
convexity (and exp-concavity \citep{hazan2007logarithmic}) to obtain
logarithmic regret bounds.

\section{Formal Definitions}
\label{section2}

Throughout this paper, $\mathbb{R}^n$ denotes a $n$-dimensional real
vector space endowed with an inner-product \(\iprod{\cdot}{\cdot}\) and
norm $\norm{\cdot}$. We take \(\Xcal \subseteq \Reals^n\) to be a fixed
convex set. The \textbf{dual norm} of $\norm{\cdot}$ is defined
by $\norm{x}_* \coloneqq \sup_{y\in\mathbb{R}^n \colon \norm{y} \leq 1}
\iprod{x}{y}$ for each \(x \in \Reals^n\). Moreover, for any convex
function \(f \colon \Xcal \to \Reals\) and any \(x \in \Reals^n\), a
vector \(g \in \Reals^n\) is a \textbf{subgradient} of \(f\) at \(x\) if
\(G\) satisfies the \emph{subgradient inequality}
\begin{equation}
    \label{eq:subgradient_ineq}
    f(z) \geq f(x) + \iprod{g}{x - z}, \qquad \forall z \in \Reals^n.
\end{equation}
We denote by \(\subdiff[f](x)\) the set of all subgradients of \(f\) at
\(x\), called the \textbf{subdifferential} of \(f\) at \(x\). The
\textbf{normal cone} of \(\Xcal\) at a point \(x \in \Xcal\) is the set
\(N_{\Xcal}(x) \coloneqq \setst{a \in \Reals^n}{\iprod{a}{z -  x} \leq
0~\text{for all}~z \in \Xcal}\).


Let  $R \colon \mathcal{D} \to \Reals $ be a convex function such that
it is differentiable in \(\Dcalcirc \coloneqq \interior{\Dcal}\) and
such that we have~\(\Xcal \subseteq \Dcalcirc\). The \textbf{Bregman divergence}
(with respect to $R$) is given by
\begin{equation*}
D_R(x,y) \coloneqq R(x)-R(y) -\iprod{\nabla R(y)}{x-y},
\qquad \forall x \in \mathcal{D}, y \in \interioro{\mathcal{D}}.
\end{equation*}
An interesting and useful identity regarding Bregman divergences,
sometimes called \emph{three-point identity}~\citep{bubeck2015convex},
is
\begin{equation}
    \label{eq:breg_identity}
D_R(x,y) + D_R(z,x) - D_R(z,y)= \iprod{\nabla R(x)-\nabla R(y)}{x-z},
\qquad \forall z \in \Dcal, \forall x,y \in \Dcalcirc.
\end{equation}
Although the Bregman divergence with respect to \(R\) is not a metric,
we can still interpret \(D_R\) as a way of measuring distances through
the lens of \(R\). An instructive example is the Bregman divergence
associated with the squared \(\ell_2\)-norm $R \coloneqq
\frac{1}{2}\norm{\cdot}_2^2$. In this case, we have
$D_R(x,y)=\frac{1}{2}\norm{x-y}_2^2$ for all \(x,y \in \Reals^n\), that
is, the divergence boils down to the squared \(\ell_2\)-distance. In
light of this, a possible way to generalize Lipschitz continuity and
strong convexity is to replace the norm in the classical definitions by
the square root of the Bregman divergence \citep{lu2018relatively}.

First, recall that a function $f \colon
\Xcal \to \Reals$ is \textbf{$L$-Lipschitz continuous} with respect to\
$\norm{\cdot}$ on \(\Xcal' \subseteq \Xcal\) if
\begin{equation*}
|f(x)-f(y)|\le L\norm{x-y}, 
\qquad \forall x, y \in \Xcal'.
\end{equation*}
 Additionally, if $f$ is convex, then the above definition
 implies\footnote{On the boundary of \(\Xcal\) this implication is not
 as strong: we can only guarantee the existence of one subgradient with
 small norm. For our purposes this will not be of fundamental
 importance. For a more precise statement see~\cite[\S
 5.3]{ben2001lectures}} that $\norm{g}_*\le L$ for all
 $x\in\mathcal{X}$ and all $g\in\partial f(x)$. Recall as well that a
 convex function $f \colon \Xcal \to \Reals$ is \textbf{$M$-strongly
 convex} with respect to\ \(\norm{\cdot}\) on \(\Xcal' \subseteq \Xcal\)
 for some $M>0$ if
\begin{equation*}
f(y) \ge f(x)+\iprod{g}{y-x}+\frac{M}{2}\norm{y-x}^2,
\qquad 
\forall x,y \in \Xcal', \forall g \in \partial f(x).
\end{equation*}
Let us now state generalizations of the above definitions due
to~\citet{lu2018relatively} and~\citet{lu2019relative}.

\begin{definition}[Relative Lipschitz continuity]
    \label{d1}
A convex function $f \colon \Xcal \to \Reals$  is \textbf{$L$-Lipschitz
continuous} relative to $R$ if
\begin{equation*}
    \label{eq:rel_lip_v0}
    \iprod{g}{x - y} 
    \leq L \sqrt{2 D_R(y,x)},
    \qquad \forall x,y \in \Xcal, \forall g \in \subdiff(x).
\end{equation*}
\end{definition}
In particular, if $f \colon \Xcal \to \Reals$  is $L$-Lipschitz
continuous relative to $R$, then 
\begin{equation}
    \label{eq:rel_lip_v2}
    f(x) - f(y) \stackrel{\eqref{eq:subgradient_ineq}}{\leq }
    \iprod{g}{x - y} 
    \leq L \sqrt{2 D_R(y,x)},
    \qquad \forall x,y \in \Xcal, \forall g \in \subdiff(x).
\end{equation}

The original definition of \citet{lu2019relative} requires \(\norm{g}_*
\norm{x - y} \leq L \sqrt{2 D_R(x,y)}\) for all \(x,y \in \Xcal\) and
\(g \in \subdiff[f](x)\). Since \(\iprod{a}{b} \leq \norm{a}_*
\norm{b}\) for any \(a,b \in \Reals^n\), the above definition is
slightly more general and does not depend on the choice of a norm.

\begin{definition}[Relative strong convexity
\citep{lu2018relatively}] A convex function $f \colon \Xcal \to
\Reals$ is \textbf{$M$-strongly convex} relative to $R$ if
\begin{equation}
    \label{eq:rel_strong_conv}
f(y) \geq f(x)+\iprod{g}{y-x}+ M D_R(y,x), 
\qquad \forall y,x \in \Xcal, \forall g \in 
\subdiff[f](x). 
\end{equation}
\end{definition}
 A notable special case of relative Lipschitz-continuity or relative
strong convexity is when we pick $R \coloneqq
\frac{1}{2}\norm{\cdot}_2^2$ and the classical definitions with respect
to the \(\ell_2\)-norm are recovered.

\paragraph{Example {\rm(A function that is relative Lipschitz but not Lipschitz).}}
Consider the function \(f\) given by $f(x) \coloneqq x^2$ for each \(x
\in \Reals\). Since the derivative of \(f\) is unbounded on \(\Reals\),
it is not Lipschitz continuous on the entire line. Define the function \(R\)
by $R(x) \coloneqq 2x^4$ for all \(x \in \Reals\). Then,
\begin{equation*}
    D_R(y,x)=2y^4-2x^4-8x^3(y-x)=\frac{1}{2}(x^2-y^2)^2+x^2(x-y)^2\ge
x^2(x-y)^2, \qquad \forall x,y \in \Reals.
\end{equation*}
Thus, $(f'(x)(x - y))^2=4x^2(x-y)^2\le 2\cdot2D_R(y,x)$ for any \(x,y
\in \Reals^n \). That is, $f$ is $\sqrt{2}$-Lipschitz continuous
relative to $R$.

\citet{lu2019relative} discusses more substantial examples in detail,
such as training of support vector machines, and finding a point in the
intersection of several ellipsoids. Furthermore, he also gives
a systematic way of picking a reference function for any objective
functions whose subgradients at \(x\) have \(\ell_2\)-norm bounded by a
polynomial in \(\norm{x}_2\).
This useful construction allows many optimization problems to benefit from algorithms that are designed for the relative setting.

\subsection{Conventions and Assumptions used Throughout the Paper}
\label{subsec:assumptions}

We collect here some additional notation and assumptions used throughout
the paper.\footnote{The only exception is
Lemma~\ref{lemma:strong_ftrl_lemma}, which does not need convexity or
differentiability of any of the functions.} First, \(\Xcal \subseteq
\Reals^n\) denotes a closed convex set and \(\{f_t\}_{t \geq 1}\)
denotes a sequence of convex functions such that \(f_t \colon \Xcal \to
\Reals\) is subdifferentiable\footnote{This is not too restrictive since
convex functions are subdifferentiable on the relative interior of their
domains~\citep[Theorem~23.4]{Rockafellar97a}.} on \(\Xcal\) for each~\(t
\geq 1\). We denote by \(\{\eta_t\}_{t \geq 0}\) a sequence of scalars
such that \(\eta_t \geq \eta_{t+1} > 0\) for each \(t \geq 0\).
Moreover, \(\Dcal \subseteq \Reals^n\) denotes a convex set with
non-empty interior \(\Dcalcirc \coloneqq \interior(\Dcal)\) such
that~\(\Xcal \subseteq \Dcalcirc\). This latter set will be the domain
of the regularizer for FTRL and of the mirror map for OMD. Namely, in
Section~\ref{section3} we denote by \(R \colon \Dcal \to \Reals\) the
\emph{regularizer} of FTRL, a convex function which is differentiable on
\(\Dcalcirc\). In Section~\ref{section4} we denote by \(\Phi \colon
\Dcal \to \Reals\) the \emph{mirror map} of online mirror descent (whose
precise definition we postpone to Section~\ref{section4}).

\section{Follow the Regularized Leader}\label{section3}

The \emph{follow the regularized leader} (FTRL) algorithm is a classical
method for OCO. At each round, FTRL picks a point that minimizes the
cost incurred by the previously seen functions plus a regularizer convex
function (an \emph{FTRL regularizer}). Intuitively, the latter helps the
choices of the algorithm not to change too widely from one round to the
next. In Algorithm~\ref{algo:FTRL} we formally outline the FTRL
algorithm. It is well known ~\citep{Hazan16a} that, in a
game with \(T\) rounds, FTRL with properly tuned step sizes suffers at
most \(O(\sqrt{T})\) regret against Lipschitz continuous
functions.\footnote{The big-O notation in this case hides constants that
may depend on the dimension and other properties of the problem at hand.
The best dependence on the Lipschitz constant and ``distance to the
comparison point'' is usually achieved when the loss functions are
Lipschitz continuous and the FTRL regularizer is strongly convex, both
with respect to the same norm.} When the loss functions are additionally
strongly convex, FTRL suffers at most regret \(O(\log T)\). In this
section we describe one of our main results: the FTRL algorithm preserves these
asymptotic regret guarantees in the relative setting.

\begin{algorithm}
    \caption{Follow the Regularized Leader (FTRL) Algorithm}
    \label{algo:FTRL}
    \begin{algorithmic}
      %
      %
      \State Compute \(x_1 \in \argmin_{x \in \Xcal} R(x)\)
      \State Set \(F_0 \coloneqq 0\)
      \For{\(t = 1,2, \dotsc\)}
      \State Observe \(f_t\) and suffer cost \(f_t(x_t)\)
      \State Set \(F_{t} \coloneqq F_{t-1} + f_t = \sum_{i = 1}^t f_i\)
       \State Compute $ x_{t+1} \in \argmin_{x \in \Xcal} \paren[\big]{
       F_t(x) + \frac{1}{\eta_t}R(x)}$
      %
      \EndFor
      %
    \end{algorithmic}
\end{algorithm}

The usual first step in the analyses of FTRL algorithms is to use basic
properties of the iterates (without relying on convexity) to bound the
algorithm's regret by easier-to-analyse terms. Such bounds are usually
the sum of two terms: the ``diameter'' of the feasible set through the
lens of the FTRL regularizer and a sum of the difference in ``quality''
between consecutive iterates. For a classic example,
see~\citep[Lemma~2.3]{Shalev-Shwartz11a}. For our analysis we shall use
a slightly tighter bound given by the Strong FTRL Lemma due to
\citet{mcmahan2017survey}. For the sake of completeness we give a proof
of this lemma (and discuss its applications in the composite setting) in
Appendix~\ref{app:strong_ftrl_lemma}.


\begin{lemma}(Strong FTRL Lemma~{\citep{mcmahan2017survey}})
    \label{lemma:strong_ftrl_lemma}
     Let $\{f_t\}_{t\ge1}$ be a sequence of functions such that $f_t
\colon \mathcal{X} \to \mathbb{R}$ for each $t\ge1$. Let
$\{\eta_t\}_{t\ge1}$ be a positive non-increasing sequence. Let $R
\colon \mathcal{X}\rightarrow\mathbb{R}$ be such that \(\curly{x_t}_{t
\geq 1}\) given as in Algorithm~\ref{algo:FTRL} is properly defined. If
\(F_t \colon \Xcal \to \Reals\) is defined as in
Algorithm~\ref{algo:FTRL} for each \(t \geq 1\), then,
    \begin{equation*}
    \Regret_T(z)
    \leq 
    \sum_{t = 0}^{T} \paren*{\frac{1}{\eta_t} - \frac{1}{\eta_{t-1}}}(R(z) - R(x_t))
     +\sum_{t = 1}^T\big(H_t(x_t)- H_t(x_{t+1})\big)
    \qquad \forall T > 0,
    \end{equation*}
    where \(\eta_0 \coloneqq 1\), \(\frac{1}{\eta_{-1}} \coloneqq 0\),
    \(x_0 \coloneqq x_1\), and \(H_t \coloneqq F_t + \frac{1}{\eta_t} R
    \) for each \(t \geq 1\).
\end{lemma}

\subsection{Sublinear Regret with Relative Lipschitz Functions}

In the following theorem we formally state our sublinear \(O(\sqrt{T})\)
regret bound of FTRL in \(T\) rounds in the setting where the cost
functions are Lipschitz continuous relative to the regularizer function
used in the FTRL method. The proof, which we defer to
Appendix~\ref{app:ftrl_sublinear regret}, boils down to bounding the
terms \(H_t(x_t) - H(x_{t+1})\) from the Strong FTRL Lemma by (roughly)
\(L^2 \eta_{t-1}/2\). We do so by combining the optimality conditions
from the definition of the iterates in Algorithm~\ref{algo:FTRL} with
the \(L\)-Lipschitz continuity relative to \(R\) of the loss functions.

\begin{theorem} 
    \label{thm:ftrl_sublinear_regret}
    Let \(\curly{x_t}_{t \geq 1}\) be defined  as in
    Algorithm~\ref{algo:FTRL} and suppose $f_t$ is $L$-Lipschitz
    continuous relative to $R$ for all~$t \geq 1$. Let \(z \in \Xcal\)
    and let $K \in \Reals$ be such that $ K \geq  R(z) - R(x_1)$. Then,
     \begin{equation*}
        \Regret_T(z)
        \leq
        \frac{K}{\eta_{T}}
        + \sum_{t = 1}^T
        \frac{L^2 \eta_{t-1}}{2}, \qquad
        \forall T > 0.
     \end{equation*}
     In particular, if \(\eta_t \coloneqq \sqrt{K}/(L\sqrt{t+1})\) for
     each \(t \geq 0\), then~\(\Regret_T(z) \leq 2 L \sqrt{K(T+1)}\).
\end{theorem}


\subsection{Logarithmic Regret with Relative Strongly Convex Functions}
\label{sec:ftrl_log_regret}

\citet{hazan2007logarithmic} showed that if the cost functions are not
only Lipschitz continuous but strongly convex as well, then the
\emph{follow the leader} (FTL) method---FTRL without any
regularizer---attains logarithmic regret. Similarly, in this section we
show that if the cost functions are relative Lipschitz continuous and
relative strongly convex, both relative to the same fixed function, then
FTL suffers regret at most logarithmic in the number of rounds. The
proof of the next theorem  is similar to the proof of
Theorem~\ref{thm:ftrl_sublinear_regret} and is deferred to
Appendix~\ref{app:log_regrets_proofs}.


\begin{theorem} 
    \label{thm:ftrl_log_regret}
    Let \(\curly{x_t}_{t \geq 1}\) be defined as in
    Algorithm~\ref{algo:FTRL} with \(R \coloneqq  0\). Assume that $f_t$
    is $L$-Lipschitz continuous and \(M\)-strongly convex relative to a
    differentiable convex function $h \colon \Dcal \to \Reals$ for each
    $t\ge1$. Then, for all \(z \in \Xcal\),
     \begin{equation*}
        \Regret_T(z)
        \leq 
        \frac{L^2}{2M} (\log(T) + 1),
        \qquad
        \forall T > 0.
     \end{equation*}
\end{theorem}

One might wonder whether requiring both Lipschitz continuity and strong convexity relative to the same function is too restritive. Indeed, let \(f\) be both \(L\)-Lipschitz continuous and \(M\)-strongly convex relative to \(R\). Moreover, assume \(f\) is differentiable for the sake of simplicity.
If \(\xstar \in \Xcal\) is a minimizer of \(f\) over \(\Xcal\), then optimality conditions imply \(- \nabla f(\xstar) \in N_{\Xcal}(\xstar)\). Thus, by the definition of relative strong convexity,
\[
f(y)-f(x^*)\ge \iprod{\nabla f(\xstar)}{y  - \xstar} + MD_R(y,x^*) \geq M D_R(y,x^*), \qquad \forall y \in \Xcal
\]
At the same time, by relative Lipschitz continuity (see~\eqref{eq:rel_lip_v2}) we have
\[
f(y)-f(x^*)\le L\sqrt{2D_R(x^*,y)}, 
\qquad \forall y \in \Xcal.
\]
This means that \(f\) has a $\Omega(D_R(\cdot,x^*))$ lower-bound and a $O(\sqrt{D_R(x^*,\cdot)})$ upper-bound. If the Bregman divergence between $y$ and $x^*$ were to go to infinity as \(y\) ranges over \(\Xcal\), for example when \(\Xcal = \Reals^n\) and \(R\) is the squared \(\ell_2\) norm, then the lower-bound would eventually exceed the upper-bound on \(\Xcal\). Therefore, relative Lipschitz continuity and relative strong convexity can only coexist when \(\Xcal\) and the Bregman divergence with respect to \(R\) of a minimizer and any point in \(\Xcal\) are both bounded. Although this is a somewhat restrictive condition, classical logarithmic regret results such as the ones due to~\citet{hazan2007logarithmic} also only hold over bounded sets. Moreover, as the next example shows, there are cases where logarithmic regret is attainable but \emph{do not} fit into classical logarithmic regret results.
\paragraph{Example {\rm(A class functions that are both relative Lipschitz continuous and relative strongly convex).}}
Define $f \coloneqq \frac{1}{p}\norm{\cdot}_2^p$ for some $p\ge2$ and suppose $\Xcal = [-\alpha, \alpha]^n$ for some \(\alpha > 0\). First, note that $\nabla f(x)=\norm{x}_2^{p-2}x$ and $\nabla^2 f(x)=\norm{x}_2^{p-2}I+(p-2)\norm{x}_2^{p-4}xx^T$ for any \(x \in \Reals^n\). By Proposition 5.1 in \citet{lu2019relative}, $f$ is 1-continuous relative to $R \coloneqq \frac{1}{2p}\norm{\cdot}_2^{2p}$ on \(\Reals^n\) since $\norm{\nabla f(x)}^2_2=\norm{x}_2^{2(p-2)}\cdot\norm{x}_2^2=\norm{x}_2^{2p-2}$. Moreover, to show that \(f\) is \(M\)-strongly convex relative to \(R\), it suffices to show that \(f - M R\) is convex (see Proposition 1.1 in \citet{lu2018relatively}). For any \(M > 0\) and \(x \in \Reals^n\) we have
\begin{align*}
    \nabla^2 f(x)-M \nabla^2 R(x)&=\norm{x}_2^{p-2}I+(p-2)\norm{x}_2^{p-4}xx^T-M(\norm{x}_2^{2p-2}I+(2p-2)\norm{x}_2^{2p-4}xx^T),\\
    &=\norm{x}_2^{p-2}(1-M\norm{x}_2^p)I+\norm{x}_2^{p-4}(p-2-M(2p-2)\norm{x}_2^p)xx^T,\\
    &\succeq\norm{x}_2^{p-4}(1-M\norm{x}_2^p+p-2-M(2p-2)\norm{x}_2^p)xx^T,\\
    &=\norm{x}_2^{p-4}(p-1-M(2p-1)\norm{x}_2^p)xx^T,
\end{align*}
where the only inequality follows since $\norm{x}_2^2I\succeq xx^T$ for any \(x \in \Reals^n\). By setting $M \coloneqq \frac{p-1}{(2p-1)(\sqrt{n} \alpha)^p}$ we have
\begin{equation*}
    p - 1 - M(2p - 1)\norm{x}_2^p \geq p - 1 - M(2p - 1)(\sqrt{n}\alpha)^p  = 0, \qquad \forall x \in \Xcal = [-\alpha, \alpha]^n.    
\end{equation*}
Thus, we have $\nabla^2 f(x)-M \nabla^2 R(x) \succeq 0$.
Therefore, \(f - M R\) is convex, which implies that  $f$ is strongly convex relative to $R$ on $\Xcal$. Note that $f$ is not classically strongly convex (that is, strongly convex with respect to the \(\ell_2\) norm) for $p\ge4$. To see this, note that \(\nabla^2 f(x) - M I\) is not positive semidefinite around \(0\) for any \(M > 0\), and thus \(f - M \norm{\cdot}_2^2\) is not convex around \(0\) no matter how small we pick \(M > 0\) to be.


\section{Dual Averaging and Composite Loss Functions}
\label{sec:da}
\label{section3.3}

FTRL is a cornerstone algorithm in OCO, but sometimes it is not practical.
Each iterate requires \emph{exact} minimization of the loss functions
(plus the regularizer) which might not have always a closed form
solution. A notable special case of FTRL that mitigates this problem is
the (online) \emph{dual averaging} (DA) method whose offline version is
due to \citet{nesterov2009primal}. 
In each iteration, DA picks a point from \(\Xcal\) that minimizes the
sum of past subgradients (scaled by the step size) plus a FTRL
regularizer $R$. Formally, for real convex functions \(\{f_t\}_{t \geq
1}\) on \(\Xcal\), the online DA method computes iterates \(\{x_t\}_{t
\geq 1}\)  such that
\begin{equation}
    \label{DA}
x_{t+1} \in \argmin_{x\in\mathcal{X}}\paren[\Big]{\eta_t\sum_{i=1}^t\langle g_i,x\rangle+R(x)}  \qquad \forall t \geq 0,
\end{equation}
where \(g_t \in \subdiff[f_t](x_t)\) for each \(t \geq 1\).
\paragraph{Intuition.} It is well-known that the DA algorithm reduces to
FTRL applied to the linearized functions \(\{\ftilde_t\}_{t \geq 1}\)
given by $\ftilde_t \coloneqq \iprod{g_t}{\cdot}$ for each \(t \in
\Naturals\) (for details see~\citet[Lemma~5.4]{Hazan16a}). This
reduction obviously preserves the property of being Lipschitz continuous
since the gradient of \(\ftilde_t\) is \(g_t\) everywhere. A natural
idea would be to use this same reduction in the relative setting.
Unfortunately, this reduction does not preserve the property of being
relative Lipschitz! Luckily, our proof only requires a weaker condition:
being ``relative Lipschitz'' at the particular point $x_t$. Namely, the
relative \(L\)-Lipschitzness (see~\eqref{eq:rel_lip_v2}) of \(f_t\)
implies \(\iprod{\nabla\ftilde_t(x_t)}{x_t - y} = \iprod{g_t}{x_t - y}
\leq L\sqrt{2D_R(y,x_t)} \) for all~\(y \in \Xcal\). That is all we need
for the proof of Theorem~\ref{thm:ftrl_sublinear_regret} to go through,
although we did state the theorem with this exact condition for the sake
of simplicity. This discussion leads to the following corollary of
Theorem~\ref{thm:ftrl_sublinear_regret}.


\begin{corollary}
    Let \(\curly{x_t}_{t \geq 1}\) be defined  as in ~\eqref{DA} and
    suppose $f_t$ is $L$-Lipschitz continuous relative to $R$ for all~$t
    \geq 1$. Let \(z \in \Xcal\) and let $K \in \Reals$ be such that $ K
    \geq  R(z) - R(x_1)$. If $\eta_t\coloneqq\sqrt{2K}/(L\sqrt{t+1})$
    for all \(t \geq 1\), then~\(\Regret_T(z) \leq 2 L \sqrt{K(T+1)}\).
\end{corollary}


Another important consideration for applications is a variant of OCO in
which the loss functions are composite~\citep{duchi2010composite,
xiao2010dual}. More specifically, in this case we have a known ``extra
regularizer''~$\Psi$, a (not necessarily differentiable) convex
function, and add it to the loss functions. The goal is to induce some
kind of structure in the iterates, such as adding
$\ell_1$-regularization to promote sparsity. Note that OCO algorithms
would still apply in this setting by replacing the loss functions
\(f_t\) with \(f_t + \Psi\) at each round \(t\). However, in this case we
are not exploiting the fact that the function \(\Psi\) is \emph{known}.
In the case of the relative setting, for example, it may be the case
that the loss functions \(f_t\) are relative Lipschitz-continuous with
respect to a certain function \(R\), while \(\Psi\) is not. In
Appendix~\ref{app:ftrl_composite} we extend the sublinear (composite)
regret bound of Theorem~\ref{thm:ftrl_sublinear_regret} and show how
this yields convergence bounds for regularized dual
averaging~\citep{xiao2010dual} in the relative setting.

\section{Dual-Stabilized Online Mirror Descent}
\label{section4}

The mirror descent algorithm is a generalization of the classical
gradient descent method that was first proposed by
\citet{nemirovsky1983problem}. A modern treatment was first given
by~\citet{beck2003mirror}. The algorithm fits almost seamlessly into the OCO setting via a variant known as online mirror descent (OMD) (see \citep{Hazan16a}).
Recently, \citet{orabona2018scale} showed that OMD with a dynamic
learning rate may suffer \emph{linear} regret.
(A dynamic learning rate is useful when we do not known the number of iterations
ahead of time.)
Moreover, this can happen even in simple and well-studied scenarios such as in the problem of prediction with expert advice, which corresponds to OMD equipped with negative entropy as a mirror map. In
general, they showed that this may happen in cases where the Bregman
divergence (with respect to the mirror map chosen) \emph{is not} bounded
over the entire feasible set. To resolve this issue, \citet{HuangHPF20a}
proposed a modified version of OMD called \emph{dual-stabilized online
mirror descent} (DS-OMD). In contrast to classical OMD, the regret bounds
for the dual-stabilized version depend only on the Bregman divergence
between the feasible set and the \emph{initial iterate}.

We formally describe the DS-OMD method in Algorithm
\ref{algorithm:DS-OMD}. Compared to OMD, DS-OMD adds an extra step in
the dual space to mix the current dual iterate with the dual of the
initial point. This step at iteration \(t\) is controlled by a
stabilization parameter~\(\gamma_t\) and it can be seen as a way to
``stabilize'' the algorithm in the dual space. Throughout this section
we closely follow the notation and assumptions
of~\citet[Chapter~4]{bubeck2015convex}. We assume that we have a
\textbf{mirror map} for \(\Xcal\), that is, a differentiable
strictly-convex function \(\Phi \colon \Dcal \to \Reals\) for \(\Xcal\)
such that the gradient of $\Phi$ diverges on the boundary of $\Dcal$,
that is, \(\lim_{x \to \partial \Dcal} \norm{\nabla \Phi(x)}_2 =
\infty\) where \(\partial \Dcal \coloneqq \Dcal \setminus
\Dcal^{\circ}\). These conditions on the mirror map guarantee that the
algorithm is well-defined (for example, they guarantee the existence
and uniqueness of the last step of Algorithm \ref{algorithm:DS-OMD}).

\begin{algorithm}
\caption{Dual-Stabilized Online Mirror Descent}
\label{algorithm:DS-OMD}
\begin{algorithmic}
\Require Stabilization coefficient $\gamma_t$ and  an initial iterate
\(x_1 \in \Xcal\).
 \For{\(t = 1,2, \dotsc\)}
    \State Observe \(f_t\) and suffer cost \(f_t(x_t)\)
    \State Compute \(g_t \in \subdiff[f_t](x_t)\)
    \State $\hat{x}_t \coloneqq \nabla\Phi(x_t)$
    \State $\hat{w}_{t+1} \coloneqq \hat{x}_t - \eta_t g_t$
    \State $\hat{y}_{t+1} \coloneqq
\gamma_t\hat{w}_{t+1}+(1-\gamma_t)\hat{x}_1$
    \State $y_{t+1} \coloneqq \nabla\Phi^*(\hat{y}_{t+1})$ 
    \State Compute $x_{t+1} \in
    \argmin_{x\in\mathcal{X}}D_\Phi(x,y_{t+1})=\Phi(x)-\Phi(y_{t+1})-\angles{\nabla\Phi(y_{t+1}),x-y_{t+1}}$
\EndFor
\end{algorithmic}
\end{algorithm}


\subsection{Sublinear Regret with Relative Lipschitz Functions}

In this section, we give a regret bound for DS-OMD when the cost
functions are all Lipschitz continuous relative to the mirror map
$\Phi$. In this setting, if we set the stabilization coefficients to be
$\gamma_t\coloneqq\eta_{t+1}/\eta_t$ and step size $O(1/\sqrt{t})$,
DS-OMD obtains sublinear regret. This is formally stated in the
following theorem.
\begin{theorem}\label{DSOMDsub}
    
     Let \(\curly{x_t}_{t \geq 1}\) be defined as in
     Algorithm~\ref{algorithm:DS-OMD} with
     $\gamma_t\coloneqq\eta_{t+1}/\eta_t$ for each \(t \geq 1\). Assume
     that $f_t$ is $L$-Lipschitz continuous relative to $\Phi$ for
     all~$t \geq 1$. Let \(z \in \Xcal\) and $K \in \Reals$ be such that
     $ K \geq D_{\Phi} (z,x_1)$. Then,
\begin{equation*}
\Regret_T(z)\le \frac{K}{\eta_{T+1}} + \sum_{t=1}^T\frac{\eta_tL^2}{2}, \qquad \forall T > 0.
\end{equation*}
In particular, if $\eta_t\coloneqq\sqrt{K}/L\sqrt{t}$ for each $t\ge1$,
then $\Regret_T(z) \le2L\sqrt{K(T+1)}$.
\end{theorem}
The proof is based on Theorem \ref{regretboundDSOMD}, which gives an abstract regret upper bound for DS-OMD. Next we compute specific upper bounds of $D_\Phi(x_t,w_{t+1})$ for each $t\ge1$ by relative Lipschitz continuity to make the abstract regret bound more specific. The whole proof of Theorem \ref{DSOMDsub} is given in Appendix~\ref{App:dsomdsublinear}.

If we set each $f_t$ to be a fixed function $f$ and take average of all iterates, then we get the following convergence rate for classical convex optimization as a corollary.
\begin{corollary}
    Let $\Phi$ be a mirror map for \(\Xcal\) and let $f \colon \Xcal \to
    \Reals$ be a convex  $L$-Lipschitz-continuous function relative to
    $\Phi$.  Let \(\curly{x_t}_{t \geq 1}\) be given as in
    Algorithm~\ref{algorithm:DS-OMD} with loss functions \(f_t \coloneqq
    f\), step sizes $\eta_t\coloneqq\sqrt{K}/L\sqrt{t}$ for some \(K
    \geq \sup_{z \in \Xcal} D_{\Phi}(z,x_1)\), and stabilization
    parameter $\gamma_t\coloneqq\eta_{t+1}/\eta_t$. If \(\xstar \in
    \Xcal\) is a minimizer of \(f\), then,
\begin{equation*}
f\parens{\frac{1}{T}\sum_{t=1}^Tx_t}-f(x^*)\le\frac{2L\sqrt{2K}}{\sqrt{T}}.
\end{equation*}
\end{corollary}
This recovers the same bound up to constant $4\sqrt{2}/3$ in Theorem 4.3 in \citet{lu2019relative}, if we take $k=T-1$ and $t_i=\frac{\sqrt{K}}{\sqrt{T}L}$ for $i\ge0$ therein.

\subsection{ Logarithmic Regret with Relative Strongly Convex Functions} 
\label{s1}
In Section~\ref{sec:ftrl_log_regret} we showed that FTRL suffers at most
logarithmic regret when the loss functions are Lipschitz continuous and
strongly convex, both relative to the same fixed reference function. Similarly, we show that OMD suffers at most logarithmic regret if we
have Lipschitz continuity and strong convexity, both relative to the mirror map $\Phi$. Interestingly, in this case the
dual-stabilization step can be skipped (that is, we can use \(\gamma_t
\coloneqq 1\) for all \(t\)) and Algorithm~\ref{algorithm:DS-OMD} boils
down to classic OMD.
\begin{theorem} \label{MD:logregret}
    Let \(\curly{x_t}_{t \geq 1}\) be given as in
    Algorithm~\ref{algorithm:DS-OMD} with \(\gamma_t \coloneqq 1\) for
    all \(t \geq 1\). Assume that $f_t$ is $L$-Lipschitz continuous and
    \(M\)-strongly convex relative to $\Phi$ for all~$t \geq 1$.
    If \(z \in \Xcal\) and \(\eta_t = \frac{1}{tM}\) for each \(t \geq
    1\), then,
\begin{equation*} 
\Regret_T(z) \le\frac{L^2}{2M}(\log T+1),\qquad\forall T > 0.
\end{equation*}
\end{theorem}
The proof involves modifications of Theorem \ref{DSOMDsub} and is deferred to Appendix \ref{A:MDlogregret}.

\vspace{-0.5em}
\subsection{Sublinear Regret with Composite Loss Functions} \label{s2}

We can extend our regret bounds to the setting with composite cost functions with minor modifications to Algorithm \ref{algorithm:DS-OMD}. The classical version OMD adapted to this setting is due to \citet{duchi2010composite} and is known by composite objective mirror descent (COMID). They showed that COMID generalizes much prior work like forward-backward splitting and derived new results on efficient matrix optimization with Schatten $p$-norms based on this framework. Details of the modification needed on Algorithm~\ref{A:DSOMD} in this setting together with regret bounds can be found in Appendix \ref{A:DSOMD}.

\vspace{-0.5em}

\section{Conclusions and Discussion}
\vspace{-0.5em}

%
%

In this paper we showed regret bounds for both FTRL and stabilized OMD
in the relative setting proposed by~\citet{lu2019relative}. All the
results hold in the \emph{anytime setting} in which we do not know the
number of rounds/iterations beforehand. Additionally, we gave logarithmic
regret bounds for both algorithms when the functions are relatively
strongly convex, analogous to the results known in the classical setting.
Finally, we extend our results to the setting of
composite cost functions, which is pervasive in practice. These results open up the possibility of a new
range of applications for OCO algorithms and may allow for new analysis
for known problems with better dependence on the instance's parameters.


At the moment there are at least two interesting directions for future
research. The first would be to investigate the connections among the
different notions of relative smoothness, Lipschitz continuity, and
strong convexity in the literature. Another is to investigate systematic ways of choosing a regularizer/mirror map for any given optimization
problem. The latter was already an interesting questions before notions
of relative Lipschitz continuity and strong convexity were proposed, but
these new ideas give more flexibility in the choice of a regularizer.

\vspace{-0.5em}
\section{Statement of Broader Impact}
\vspace{-0.5em}
In this paper we study the performance of online convex optimization
algorithms when the functions are not necessarily Lipschitz continuous,
a requirement in classical regret bounds. This opens up the range of
applications for which we can use OCO with good guarantees and guides
how such parameters such as regularizers/mirror maps and step sizes
should be chosen. It is our hope that this aids practitioners to develop
more efficient ways to optimize and train their current models.
Furthermore, we hope theoreticians to be inspired to delve deep into the
setting of non-smooth optimization beyond Lipschitz continuity. It not
only opens up the range of applications, but sheds light onto the
fundamental conditions on the cost functions and regularizers/mirror
maps needed for OCO algorithms to have good guarantees. Due to the
theoretical nature of this work, we do not see potentially bad societal
or ethical impacts.
\vspace{-0.5em}
\section*{Acknowledgments}
\vspace{-0.5em}
We would like to thank the three anonymous reviewers and the meta-reviewer for engaging with our work. Moreover, we are thankful for their useful suggestions regarding the logarithmic regret results and the relationship of relative Lipschitz continuity and  Riemann-Lipschitz continuity \citep{antonakopoulosonline}. We are also thankful to Wu Lin for useful discussions during the development of this work. Finally, we are grateful to Francesco Orabona for identifying in the work of~\citet{HazanAK07a} some relationship with our results and one of the first uses of relative strong convexity.
\vspace{-0.5em}
\section*{Funding Disclosure}

This research was partially supported by 
NSERC Discovery Grants, Canada Research Chairs, the CIFAR Learning in Machines and Brains program, and the Canada CIFAR AI Chair Program.

\bibliographystyle{abbrvnat}
\bibliography{bib}


\newpage

\appendix

\renewcommand{\thefootnote}{\fnsymbol{footnote}}

\footnotetext[1]{Equal contributions.}

\section{Relationship with Riemann-Lipschitz Continuity}
\label{sec:RLC}
\citet{antonakopoulosonline} introduced the idea of \emph{Riemann-Lipschitz} continuity (RLC). They show how FTRL and OMD can be used when the cost functions are all RLC in a way that guarantees \(O(\sqrt{T})\) regret. In this section we shall discuss the relationship between these two generalizations of Lipschitz continuity. Ultimately, we will see that our results are at least as general but that further study into the relationship between these ideas is needed. We note that we will closely follow the notation of \citet{antonakopoulosonline} and shall not discuss Riemannian metrics in full generality.

Let \(G \colon \Reals^n \to \Reals^{n \times n}\) be such that \(G(x)\) is a symmetric positive definite matrix for all \(x \in \Xcal\setminus\{0\}\) and \(G(0)\) is symmetric positive semidefinite. Then the \textbf{Riemannian metric} (induced by G) is the collection of bilinear pairings \(\setst{\iprod{\cdot}{\cdot}_x}{x \in \Xcal}\) defined by
\begin{equation*}
    \iprod{y}{z}_x \coloneqq \iprodt{y}{G(x)z}, \qquad \forall x,y,z \in \Xcal.
\end{equation*}
For conciseness, we shall denote the above metric induced by \(G\) simply as the metric \(G\). Moreover, the local norm induced by such the metric \(G\) on \(x \in \Xcal\) is naturally given by
\begin{equation*}
    \norm{z}_x \coloneqq \sqrt{\iprod{z}{G(x) z}},
    \qquad \forall z \in \Xcal.
\end{equation*}
Let us now give the definition of Riemann-Lispchitz continuity. 
\begin{definition}
Let \(L > 0\). A function $f \colon \mathcal{X} \to \mathbb{R}$ is \(L\)-\textbf{Riemann-Lipschitz continuous} (RLC) relative to a Riemannian metric $G$  if
\[
|f(y)-f(x)|\le L \cdot \mathrm{dist}_G(x,y)\qquad \forall x,y \in \Xcal,
\]
where $\text{dist}_G(x,y)$ is the Riemannian distance\footnote{We do not give here the full definition of a Riemannian metric as given by \citet{antonakopoulosonline} since it will not be used in any of our discussions.} between $x$ and $y$ induced by the Riemannian metric~$G$.
\end{definition}
The above definition is notably hard to work with. In the case of differentiable functions, RLC boils down to a much simpler and more intuitive condition.

\begin{proposition}[{\citep[Proposition~1]{antonakopoulosonline}}]
\label{prop:RLC}
Suppose that $f\colon\mathcal{X} \to \mathbb{R}$ is differentiable. Then \(f\) is \(L\)-RLC  if and only if
\begin{equation}
\label{eq:rlc_def}
\norm{\grad f(x)}_x\le L\qquad\text{ for all }x\in\mathcal{X},
\end{equation}
where\footnote{Here we overlook the case when \(x = 0\)   (and, thus, when \(G(x)\) is not necessarily invertible), for the sake of simplicity.} $\grad f(x) \coloneqq G(x)^{-1} \nabla f(x)$ is the Riemannian gradient of $f$ at $x$ with respect to the metric~\(G\).
\end{proposition}

Finally, \citet{antonakopoulosonline} use the notion of a \emph{strong convexity} of a closed convex function \(R \colon \Xcal \to \Reals\) with respect to a metric \(G\). For the sake of conciseness and simplicity, we shall use the equivalent condition given by~\citet[Lemma~1]{antonakopoulosonline} and assume that \(R\) is differentiable, but the arguments of this section hold even if \(R\) is a closed convex function with a continuous selection of subgradients. More specifically, a differentiable convex function \(R\) is \(K\)-\textbf{strongly convex} with respect to the metric \(G\) for \(K > 0\) if
\begin{equation*}
    \frac{K}{2}\norm{x - y}_x^2
    \leq D_{R}(y,x), \qquad \forall x,y \in \Xcal.
\end{equation*}

We are now in place to discuss the relationship between the notions of relative Lipchitz continuity and RLC. First, one should note that Proposition~\ref{prop:RLC} requires differentiability to hold. Since the regret bounds in~\citet{antonakopoulosonline} rely on~\eqref{eq:rlc_def}, they also rely on the cost functions being differentiable. Since most \(O(\sqrt{T})\) regret bounds in the online convex optimization literature (as well as the regret bounds in this text) \emph{do not} rely on differentiability  of the cost functions, it would be interesting to investigate if differentiability of the cost functions is in fact needed for the regret bounds of~\citet{antonakopoulosonline} to hold. In particular, in a way similar to classic Lipschitz continuity, it might be the case that~\eqref{eq:rlc_def} holds for at least one subgradient (after transformation by the metric \(G\)) at each point \(x \in \Xcal\) in the non-differentiable case. 

Assuming that the cost functions are indeed differentiable, we can show that relative Lipschitz continuity is at least as general as RLC. In the following proposition we show that if \(f\) is a RLC function with respect to a metric \(G\) and if we have a differentiable convex function \(R\) which is strongly convex w.r.t.\ \(G\) (which is used as a regularizer or a mirror map in FTRL and OMD), then \(f\) is Lipschitz continuous relative to \(R\).
\begin{proposition}
    Let \(f \colon \Xcal \to \Reals\) be a differentiable convex function and let \(R \colon \Xcal \to \Reals\) be a differentiable convex function such that \(R\) is \(K\)-strongly convex with respect to the Riemannian metric \(G\). If \(f\) is \(L\)-RLC with respect to \(G\), then \(f\) is \(L'\)-Lipschitz continuous relative to \(R\) where we set~\(L' \coloneqq L \sqrt{K/2}\).
\end{proposition}
\begin{proof}
    Let \(x \in \Xcal\). First, note that
    \begin{align*}
        \norm{\grad f(x)}_x^2 
        &= \iprodt{\grad f(x)}{G(x) \grad f(x)}
        = \iprodt{\nabla f(x)}{G(x)^{-1}G(x) G(x)^{-1} \nabla f(x)}\\
        &= \iprodt{\nabla f(x)}{G(x)^{-1} \nabla f(x)}
        =\norm{\nabla f(x)}_{x,*}^2,
    \end{align*}
    where \(\norm{\cdot}_{x,*}\) is the dual norm of \(\norm{\cdot}_{x}\). Therefore, for any \(y \in \Xcal\),
    \begin{align*}
        \iprodt{\nabla f(x)}{(x - y)}
        &\leq \norm{\nabla f(x)}_{x, *}
        \norm{x - y}_x
        &(\text{by the definition of dual norm}),
        \\
        &\leq L \norm{x - y}_x, 
        &(\text{by RLC}),
        \\
        &\leq L \sqrt{\frac{K}{2} D_R(y, x)}, 
        &(\text{by strong convexity of \(R\) w.r.t.\ \(G\)}). &\qedhere
    \end{align*}
\end{proof}

The above proposition shows that Riemann-Lipschitz continuity (together with a strongly convex function with respect to the Riemannian metric) implies relative Lipschitz continuity. Thus, our regret bounds can be seen as generalizations of the regret bounds due to \citet{antonakopoulosonline}. Moreover, the modularity of our proofs makes it easier to extend the results to the different settings (as demonstrated to the extension of some regret bounds to the composite setting as shown in Section~\ref{sec:da}, for example ).

Regarding the implication in the other direction, that is, whether relative Lipschitz continuity implies Riemannian Lipschitz continuity with respect to some metric \(G\), it is not clear if it holds in general. The problem is that we do not know a systematic way of obtaining a metric \(G\) given a function \(f\) Lipschitz continuous relative to a function \(R\) such that \(f\) is RLC with respect to \(G\) \emph{and} \(R\) is strongly convex with respect to \(G\). Still, in some examples such a metric \(G\) does seem to exist. It is not clear at the moment if both concepts of Lipschitz continuity are equivalent or not.

\section{Arithmetic Inequalities}
\label{app:arithmetic_ineq}
\begin{lemma}
    \label{A1}
    Let \(\{a_t\}_{t \geq 1}\) be a non-negative sequence with \(a_1 >
    0\). Then,
    \begin{equation*}
        \sum_{t = 1}^T \frac{a_t}{\sqrt{\sum_{i = 1}^t a_i}} 
        \leq 2 \sqrt{\sum_{t = 1}^T a_t},
        \qquad \forall T \in \Naturals.
    \end{equation*}
\end{lemma}
\begin{proof}
    The proof is by induction on \(T\). The statement holds trivially
  for \(T = 1\). Let \(T > 1\) and define \(s \coloneqq \sum_{t = 1}^T
  a_t\). By the induction hypothesis,
  \begin{equation*}
    \sum_{t = 1}^T \frac{a_t}{\sqrt{\sum_{i = 1}^t a_i}} \leq
    2 \sqrt{\sum_{t = 1}^{T -1} a_t} + \frac{a_T}{\sqrt{\sum_{i = 1}^T
        a_i}}
    = 2 \sqrt{s - a_T} + \frac{a_T}{\sqrt{s}}.
  \end{equation*}
  Finally, note that
  \begin{align*}
    2 \sqrt{s - a_T} + \frac{a_T}{\sqrt{s}} \leq 2 \sqrt{s}
    &\iff 2 \sqrt{s(s - a_T)} \leq 2 s - a_T
    \iff 4 s(s - a_T) \leq (2 s - a_T)^2,\\
    &\iff 4s^2 - 4s a_T \leq 4 s^2 - 4 s a_T + a_T^2
    \iff 0 \leq a_T^2. \qedhere
  \end{align*}
\end{proof}

\section{Proofs for Section~\ref{section3}}
\label{app:proofs_sec3}

\subsection{Strong FTRL Lemma}
\label{app:strong_ftrl_lemma}

In this section we give a proof of Lemma~\ref{lemma:strong_ftrl_lemma}
for completeness. We also show how the lemma can be used for the
composite setting. For further discussions on the lemma and on FTRL, see
the thorough survey of~\citet{mcmahan2017survey}.

\begin{proof}[Proof of Lemma~\ref{lemma:strong_ftrl_lemma}] 
    
    Fix \(T > 0\). Define \(r_t \coloneqq (\frac{1}{\eta_t} -
    \frac{1}{\eta_{t-1}})R\) for each \(t \geq 0\) (recall that \(\eta_0
    \coloneqq 1\) and \(1/\eta_{-1} \coloneqq 0\)), define \(h_{t}
    \coloneqq r_{t} + f_t\) for each \(t \geq 1\), and set \(h_0
    \coloneqq r_0\). In this way, we have 
    \begin{equation*}
        \sum_{i = 0}^t h_t = \sum_{i = 1}^t f_t + \sum_{i = 0}^t r_t
        = \sum_{i = 1}^t f_t + \frac{1}{\eta_t} R = H_t, \qquad \forall t \geq 0.
    \end{equation*}
    In particular,
    \begin{equation}
      \label{eq:str_ftrl_2}
      x_t \in \argmin_{x \in \Xcal} H_{t-1}(x) = \argmin_{x \in \Xcal} \sum_{i = 0}^{t-1}
      h_i(x), \qquad \forall t \geq 0.
   \end{equation}
   Let us now bound the regret of the points \(x_1, \dotsc, x_T\) with
   respect to the functions \(h_1, \dotsc, h_{T}\) and to a comparison
   point \(z \in \Xcal\) (plus a \(-h_0(z)\) term):
    \begin{align*}
      \sum_{t = 1}^{T} (h_{t}(x_{t}) - h_{t}(z)) - h_0(z)
      &= \sum_{t = 1}^{T} h_{t}(x_{t}) - H_{T}(z)
        = \sum_{t = 1}^T (H_{t}(x_t) - H_{t-1}(x_t)) - H_{T}(z),
      \\ &\stackrel{\eqref{eq:str_ftrl_2}}{\leq}
      \sum_{t = 1}^T (H_{t}(x_t) - H_{t-1}(x_t)) - H_{T}(x_{T+1}),\\
      &= \sum_{t = 1}^T (H_{t}(x_t) - H_{t}(x_{t+1})) - H_0(x_{1}),
    \end{align*}
    where in the last equation we just re-indexed the summation, placing
    \(H_{T+1}(x_{T+1})\) inside the summation, and leaving \(H_0(x_1)\)
    out. Re-arranging the terms and using \(H_0 = h_0 = r_0\) and
    \(x_0 = x_1\) yield
    \begin{align*}
      \sum_{t = 1}^T (f_t(x_t) + r_{t}(x_t) - f_t(z) - r_{t}(z))
      &=\sum_{t = 1}^T (h_{t}(x_t) - h_{t}(z)),
      \\ &\leq r_0(z) - r_0(x_0) +  \sum_{t = 1}^T
           (H_{t}(x_t) -  H_{t}(x_{t+1})),
    \end{align*}
    which implies
    \begin{equation*}
      \Regret_T(z)
      = \sum_{t = 1}^T (f_t(x_t) - f_t(z) )
      \leq
      \sum_{t = 0}^{T} (r_{t}(z)
      - r_{t}(x_{t})) + \sum_{t = 1}^T (H_{t}(x_t)
      -  H_{t}(x_{t+1})).
    \end{equation*}
    Since \(r_t = (\frac{1}{\eta_t} - \frac{1}{\eta_{t-1}})R\) for all
    \(t \geq 0\), we have
    \begin{align*}
        \sum_{t = 0}^{T} (r_{t}(z)
      - r_{t}(x_{t})) 
      = \sum_{t = 0}^{T} \paren[\Big]{\frac{1}{\eta_t} - \frac{1}{\eta_{t-1}}}(R(z)
      - R(x_{t})).&\qedhere
    \end{align*}
\end{proof}

For the composite setting (see Section~\ref{app:ftrl_composite}), we
modify the definition of \(r_t\) for \(t \geq 1\) (maintaining the
definition of \(r_0\)) in the above proof for
\begin{equation*}
    r_t \coloneqq \paren[\Big]{\frac{1}{\eta_t} - \frac{1}{\eta_{t-1}}}R + \Psi, \qquad \forall t \geq 1.
\end{equation*}
In this case, we have
\begin{equation*}
    H_t = \sum_{i = 1}^t f_t + \sum_{i = 0}^t r_t
    =\sum_{i = 1}^t f_t + \frac{1}{\eta_t} R + t \Psi.
\end{equation*}
Proceeding in the same way as in the proof of
Lemma~\ref{lemma:strong_ftrl_lemma}, we get
\begin{align*}
    \sum_{t = 1}^T(f_t(x_t) - f(z))
      \leq
      &\sum_{t = 0}^{T} \paren[\Big]{\frac{1}{\eta_t} - \frac{1}{\eta_{t-1}}}(R(z)
      - R(x_{t})),\\
      &+ \sum_{t = 1}^T(\Psi(z) - \Psi(x_t)) + \sum_{t = 1}^T (H_{t}(x_t)
      -  H_{t}(x_{t+1})),
\end{align*}
Re-arranging yields
\begin{equation}
    \label{eq:composite_strong_ftrl_lemma}
    \Regret_T^{\Psi}(z) \leq \sum_{t = 0}^{T} \paren[\Big]{\frac{1}{\eta_t} - \frac{1}{\eta_{t-1}}}(R(z)
    - R(x_{t}))\\
     + \sum_{t = 1}^T (H_{t}(x_t)
    -  H_{t}(x_{t+1})).
\end{equation}

\subsection{Sublinear Regret with Relative Lipschitz Functions}
\label{app:ftrl_sublinear regret}

With the Strong FTRL Lemma, to derive regret bounds we can focus on
bounding the difference in cost between consecutive iterates. In this
section we will prove the sublinear regret bound for FTRL from
Theorem~\ref{thm:ftrl_sublinear_regret}. In the next lemma we give a
bound on these costs based on the Bregman divergence of the FTRL
regularizer, this time relying on convexity (but not on much more).
Loosely saying, the first claim of the next lemma follows from the
optimality conditions of the iterates of FTRL and the second follows
from the subgradient inequality.

\begin{lemma}
    \label{lemma:opt_cond_ftrl}
     Let \(\curly{x_t}_{t \geq 1}\) and $\{F_t\}_{t \geq 0}$ be defined
     as in Algorithm~\ref{algo:FTRL}. Then, for each \(t \in \Naturals\)
     there is \(p_t \in N_{\Xcal}(x_t)\) such that $- p_t -
     \frac{1}{\eta_{t-1}}\nabla R(x_t) \in \subdiff[F_{t-1}](x_t)$,
     where \(\eta_0 \in \Reals\) can be any positive constant. Moreover,
     this implies
    \begin{equation*}
        \begin{aligned}
            F_{t-1}(x_t)
        - F_{t-1}(x_{t+1})
        \leq \frac{1}{\eta_{t-1}}
        \paren[\big]{R(x_{t+1}) - R(x_{t}) - D_R(x_{t+1}, x_t)}.
        \end{aligned}
    \end{equation*}  
\end{lemma}
\begin{proof}
    Let \(t \geq 1\). By the definition of the FTRL algorithm, we have
    \(x_t \in \argmin_{x \in \Xcal}(F_{t-1}(x) + \frac{1}{\eta_{t-1}}
    R(x))\). By the optimality conditions for convex programs, we have
    \begin{equation*}
        \subdiff[\paren*{F_{t-1} + \tfrac{1}{\eta_{t-1}} R}](x_t) \cap (- N_{\Xcal}(x_t)) \neq \emptyset.
    \end{equation*}
    Since \(\subdiff[\paren*{F_{t-1} + \tfrac{1}{\eta_{t-1}} R}](x_t) =
    \subdiff[F_{t-1}](x_t) + \tfrac{1}{\eta_{t-1}} \nabla R(x_t)\), the
    above shows there is \(p_t \in N_{\Xcal}(x_t)\) such that
    \begin{equation*}
        - p_t - \frac{1}{\eta_{t-1}}\nabla R(x_t)
        \in \subdiff[F_{t-1}](x_t).
    \end{equation*}
    Using the subgradient inequality~\eqref{eq:subgradient_ineq} with
    the above subgradient yields,
    \begin{align*}
        &F_{t-1}(x_t) - F_{t-1}(x_{t+1})\\
        &\leq -\iprod{p_t}{x_t - x_{t+1}}
        - \tfrac{1}{\eta_{t-1}} \iprod{\nabla R(x_t)}{x_t - x_{t+1}},
        \\
        &\leq 
        - \tfrac{1}{\eta_{t-1}} \iprod{\nabla R(x_t)}{x_t - x_{t+1}}
        & \text{(by the definition of normal cone),} 
        \\
        &=\tfrac{1}{\eta_{t-1}}
        \paren[\big]{R(x_{t+1}) - R(x_{t}) - D_R(x_{t+1}, x_t)},
    \end{align*}
    where in the last equation we used that, by definition of the
    Bregman divergence, \(D_R(x_{t+1}, x_t) = R(x_{t+1}) - R(x_t) -
    \iprod{\nabla R(x_t)}{x_{t+1} - x_t}\) and, thus, \( - \iprod{\nabla
    R(x_t)}{x_t - x_{t+1}} = R(x_{t+1}) - R(x_t) - D_R(x_{t+1}, x_t)\).
\end{proof}

\begin{proof}[Proof of Theorem~\ref{thm:ftrl_sublinear_regret}]
    
    For each \(t \geq 0\) let \(H_t\) be defined as in the Strong FTRL
    Lemma and fix \(t \geq 0\). We have
    \begin{equation}
        \label{eq:ftrl_bound_break}
        H_t(x_t)-H_t(x_{t+1}) = F_t(x_t) - F_t(x_{t+1})
        +\frac{1}{\eta_t}(R(x_t)-R(x_{t+1})).
    \end{equation}
    Using \(F_t = F_{t-1} + f_t\) together with
    Lemma~\ref{lemma:opt_cond_ftrl} we have
    \begin{align*}
        F_t(x_t) - F_t(x_{t+1})
        &= F_{t-1}(x_t) - F_{t-1}(x_{t+1}) 
        + f_t(x_t) - f_t(x_{t+1})
        ,\\
        &\leq \frac{1}{\eta_{t-1}}
        \paren[\big]{R(x_{t+1}) - R(x_{t}) - D_R(x_{t+1}, x_t)}
        + f_t(x_t) - f_t(x_{t+1}).
    \end{align*}
    Plugging the above inequality onto \eqref{eq:ftrl_bound_break} yields
    \begin{equation}
        \label{eq:ftrl_bound_2}
        \eqref{eq:ftrl_bound_break}
        \leq 
        f_t(x_t) - f_t(x_{t+1}) 
        - \frac{D_R(x_{t+1}, x_t)}{\eta_{t-1}}
        + \paren[\Big]{\frac{1}{\eta_t} - 
        \frac{1}{\eta_{t-1}}}(R(x_t) - R(x_{t+1})).
    \end{equation}
    Since \(f_t\) is \(L\)-relative Lipschitz continuous with respect to\
    \(R\), we apply~\eqref{eq:rel_lip_v2} followed by the the
    arithmetic-geometric mean inequality \(\sqrt{\alpha \beta} \leq(\alpha +
    \beta)/2\)  with \(\alpha \coloneqq L^2 \eta_{t-1}\) and \(\beta
    \coloneqq 2D_R(x_{t+1}, x_t)/\eta_{t-1}\) to get
    \begin{equation*}
        f_t(x_t) - f_t(x_{t+1}) 
        - \frac{D_R(x_{t+1}, x_t)}{\eta_{t-1}}
        \stackrel{\eqref{eq:rel_lip_v2}}{\leq} L \sqrt{2 D_R(x_{t+1}, x_{t})} 
        - \frac{D_R(x_{t+1}, x_t)}{\eta_{t-1}} \leq \frac{L^2\eta_{t-1}}{2}.
    \end{equation*}
    Applying the above on \eqref{eq:ftrl_bound_2} yields
    \begin{equation*}
        \eqref{eq:ftrl_bound_2}
        \leq \frac{L^2\eta_{t-1}}{2}  +\paren[\Big]{\frac{1}{\eta_t} - 
        \frac{1}{\eta_{t-1}}}(R(x_t) - R(x_{t+1})).
    \end{equation*}
    Plugging the above inequality into the the Strong FTRL Lemma together
    with \(R(x_1) \leq R(x_t)\) for each~\(t \geq 1\) (which follows by the definition of \(x_1\)) yields
    \begin{align*}
        \Regret_T(z)
        &\leq 
        \sum_{t = 0}^{T} \paren[\Big]{\frac{1}{\eta_t} - \frac{1}{\eta_{t-1}}}(R(z) - R(x_t) + R(x_t) - R(x_{t+1}))
         +\sum_{t = 1}^T\frac{L^2\eta_{t-1}}{2}
        ,\\
        &=\sum_{t = 0}^{T} \paren[\Big]{\frac{1}{\eta_t} - \frac{1}{\eta_{t-1}}}(R(z)  - R(x_{t+1}))
        +\sum_{t = 1}^T\frac{L^2\eta_{t-1}}{2},
        \\
        &\leq  \frac{1}{\eta_{T}}(R(z) - R(x_1))
        + \sum_{t = 1}^T
        \frac{L^2 \eta_{t-1}}{2}
        \leq
        \frac{K}{\eta_{T}}
        + \sum_{t = 1}^T
        \frac{L^2 \eta_{t-1}}{2}.
    \end{align*}
    If we set \(\eta_t \coloneqq \sqrt{2K}/(L\sqrt{t+1})\) and since
    \(\sum_{t = 1}^T \frac{1}{\sqrt{t}} \leq 2 \sqrt{T}\) by
    Lemma~\ref{A1} in Appendix~\ref{app:arithmetic_ineq}, then
    \begin{equation*}
        \Regret_T(z)
        \leq L \sqrt{K(T+1)}
        + \frac{L \sqrt{K}}{2}\sum_{t = 1}^T \frac{1}{\sqrt{t}}
        \leq L \sqrt{K(T+1)}
        + L \sqrt{K T}
        \leq 2 L \sqrt{K(T+1)}. \qedhere
    \end{equation*}
    \end{proof}

\subsection{Logarithmic Regret}
\label{app:log_regrets_proofs}

The next lemma strengthens the bound from
Lemma~\ref{lemma:opt_cond_ftrl} in the case where the loss functions are
relative strongly convex with respect to a fixed reference function. We
further simplify matters by taking \(R = 0\), that is, regularization is
not needed for FTRL in the relative strongly convex case.

\begin{lemma}
    \label{lemma:opt_cond_strong_conv_ftrl}
    %
    Let \(\curly{x_t}_{t \geq 1}\) be defined as in
    Algorithm~\ref{algo:FTRL} with \(R \coloneqq 0\). Moreover, let $h
    \colon \Dcal \to \Reals$ be a differentiable convex function such
    that $f_t$ is $M$-strongly convex relative to \(h\)  for each
    $t\ge1$. Then, for all \(T \geq 1\),
    \begin{equation*}
        \begin{aligned}
        F_{t-1}(x_t)
        - F_{t-1}(x_{t+1})
        &\leq 
        -(t-1)MD_h(x_{t+1}, x_t).
        \end{aligned}
    \end{equation*}  
\end{lemma}
\begin{proof}
    Let \(t \geq 1\). Note that \(F_{t-1}\) is \((t-1)M\)-strongly
    convex relative to \(R\) since it is the sum of \(t-1\) functions
    that are each \(M\)-strongly convex relative to \(R\). Additionally,
    let \(p_t \in N_{\Xcal}(x_t)\) be as given by
    Lemma~\ref{lemma:opt_cond_ftrl}. By this lemma we have~\(-p_t  \in
    \subdiff[F_{t-1}](x_t)\). Thus, using
    inequality~\eqref{eq:rel_strong_conv} from the definition of
    relative strong convexity with this subgradient yields
    \begin{equation*}
        \label{eq:log_lemma_1}
        F_{t-1}(x_t) - F_{t-1}(x_{t+1})
        \leq -\iprod{p_t}{x_t - x_{t+1}}
        -(t-1)MD_h(x_{t+1}, x_t).
    \end{equation*}
    By the definition of normal cone we have \(-\iprod{p_t}{x_t -
    x_{t+1}} = \iprod{p_t}{x_{t+1} - x_t} \leq 0\), which yields the desired inequality.
\end{proof}

\begin{proof}[Proof of Theorem~\ref{thm:ftrl_log_regret}] 
     For each \(t \geq 0\) let
    \(H_t \colon \Xcal \to \Reals\) be defined as in the Strong FTRL
    Lemma and fix \(t \geq 0\). Since \(R = 0\), we have \(H_t = F_t\).
    This together with Lemma~\ref{lemma:opt_cond_strong_conv_ftrl}
    yields 
    \begin{align}
        H_t(x_t)-H_t(x_{t+1})
        &= F_t(x_t) - F_t(x_{t+1})
        = F_{t-1}(x_t) - F_{t-1}(x_{t+1}) + f_t(x_t) - f_t(x_{t+1})
        \nonumber
        ,\\
        &\leq -(t-1)M D_h(x_{t+1}, x_t) + f_t(x_t) - f_t(x_{t+1}).
        \label{eq:ftrl_bound_1}
    \end{align}
    Let \(g_t \in \subdiff[f_t](x_t)\). Since \(f_t\) is \(L\)-Lipschitz
    continuous and \(M\)-strongly convex, both relative to \(h\), we
    have
    \begin{equation*}
        f_t(x_t) - f_t(x_{t+1})
        \stackrel{\eqref{eq:rel_strong_conv}}{\leq}
        \iprod{g_t}{x_t - x_{t+1}}
        - M D_h(x_{t+1}, x_t)
        \stackrel{\eqref{eq:rel_lip_v2}}{\leq} L \sqrt{2 D_R(x_{t+1}, x_{t})} 
        - MD_R(x_{t+1}, x_t).
    \end{equation*}
    Applying the above to \eqref{eq:ftrl_bound_1} together with the fact
    that \(\sqrt{\alpha \beta} \leq(\alpha + \beta)/2\) with \(\alpha
    \coloneqq L^2/(Mt)\) and \(\beta \coloneqq 2 tM D_R(x_{t+1}, x_t)\)
    yields 
     \begin{equation*}
        H_t(x_t) - H_t(x_{t+1})
        \leq  L \sqrt{2 D_R(x_{t+1}, x_{t})} 
        - tMD_R(x_{t+1}, x_t)
        \leq \frac{L^2}{2Mt}.
     \end{equation*}
     Finally, plugging the above inequality into the Strong FTRL Lemma
     (with \(R = 0\)) gives
    \begin{align*}
        \Regret_T(z) 
        &\leq
        \sum_{t = 0}^{T} \paren*{H_t(x_t) - H_t(x_{t+1})}
        \leq 
        \frac{L^2}{2M}
      \sum_{t = 1}^T
        \frac{1}{t}
       \leq 
       \frac{L^2}{2M} (\log(T) + 1). \qedhere
    \end{align*}
\end{proof}


\section{Sublinear Regret Bounds for FTRL with Composite Loss Functions}
\label{app:ftrl_composite}

In this section we extend the results from Section~\ref{section3} to the
case where the loss functions are \emph{composite}. Specifically, there
is a known non-negative convex function \(\Psi \colon \Xcal \to
\Reals_+\) (sometimes called \emph{extra regularizer}) which is
subdifferentiable
 on \(\Xcal\) and at round \(t\) the loss function presented to the
player is \(f_t + \Psi\). Usually \(\Psi\) is a simple function which is
easy to optimize over (such as the \(\ell_1\)-norm). Thus, although
\(f_t + \Psi \) might not preserve relative Lipschitz continuity of
\(f_t\), one might still hope to obtain good regret bounds in this case.
We shall see that FTRL does not need any modifications to enjoy of good
theoretical guarantees in this setting. Yet, its analysis in the
composite case will allow us to derive regret bounds for the
\emph{regularized dual averaging} method due to~\citet{xiao2010dual}.

In the composite case we measure the performance of an OCO algorithm by
its \textbf{composite regret} (against a point \(z \in \Xcal\)) given by
\begin{equation}
    \label{eq:composite_regret}
    \Regret_T^{\Psi}(z)
    \coloneqq \sum_{t = 1}^T (f_t(x_t) + \Psi(x_t)) - \inf_{z \in \Xcal}
    \sum_{t = 1}^T (f_t(z) + \Psi(z)),
    \qquad \forall T > 0.
\end{equation}

In the case of FTRL, practically no modifications to the algorithm are
needed. Namely, the update of Algorithm \ref{algo:FTRL} becomes
\begin{equation*}
    x_{t+1}
    \in \argmin_{x\in\mathcal{X}} \paren[\Big]{\sum_{i=1}^t f_i(x) + t\Psi(x) +
    \frac{1}{\eta_t}R(x)},
    \qquad \forall t \geq 0.
\end{equation*}
We do make the additional assumption that \(\Psi(x_1) = 0\), that is,
\(x_1\) minimizes \(\Psi\) and tha latter has minimum value of \(0\). In
practice one has some control on \(\Psi\), so this assumption is not too
restrictive. The next theorem shows that we can recover the regret bound
from Theorem~\ref{thm:ftrl_sublinear_regret} for the composite setting
even if \(\Psi\) is not relative Lipschitz-continuous with respect to
the FTRL regularizer.
\begin{theorem}
    \label{theorem8}
    Let \(\Psi \colon \Xcal \to \Reals_+\) be a nonnegative convex
    function such that \(\curly{x_t}_{t \geq 1}\) as given as in
    Algorithm~\ref{algo:FTRL} are such that \(\Psi(x_1) = 0\). Assume
    that $f_t$ is $L$-Lipschitz continuous relative to $R$ for all~$t
    \geq 1$. Let \(z \in \Xcal\) and \(K \in \Reals\) be such that $ K
    \geq R(z) - R(x_1)$. Additionally, assume \(\Psi(x_1) = 0\). Then,
     \begin{equation*}
        \Regret_T^{\Psi}(z)
        \leq
        \frac{2K}{\eta_{T}}
        + \sum_{t = 1}^T
        \frac{L^2 \eta_{t-1}}{2}, \qquad
        \forall T > 0.
     \end{equation*}
     In particular, if \(\eta_t \coloneqq \sqrt{2K}/(L\sqrt{t+1})\) for
     each \(t \geq 1\), then~\(\Regret_T^{\Psi}(z) \leq 2 L
     \sqrt{K(T+1)}\)
\end{theorem}

The proof is largely identical to the proof of Theorem
\ref{thm:ftrl_sublinear_regret}. One of the main differences in the
analysis is the following version of Lemma \ref{lemma:opt_cond_ftrl}
tweaked for the composite setting. It follows by adding \((t-1)\Psi\) to
\(F_{t-1}\) in the proof of the original lemma and using the properties
of the subgradient. We give the full proof for the sake of completeness.
\begin{lemma}\label{lemma9}
    Let \(\Psi \colon \Xcal \to \Reals_+\) be a nonnegative convex
    function such that \(\curly{x_t}_{t \geq 1}\) as given as in
    Algorithm~\ref{algo:FTRL} are such that \(\Psi(x_1) = 0\). Then, for
    each \(t \in \Naturals\) there is \(p_t \in N_{\Xcal}(x_t)\) such
    that
    \begin{equation*}
        - p_t - \frac{1}{\eta_{t-1}}\nabla R(x_t)
        \in \partial\big(F_{t-1}+(t-1)\Psi\big)(x_t),
    \end{equation*}
    and the above implies
    \begin{equation*}
        \begin{aligned}
            &F_{t-1}(x_t)
        - F_{t-1}(x_{t+1})
        + (t-1)(\Psi(x_{t})-\Psi(x_{t+1}))\\
        &\leq \frac{1}{\eta_{t-1}}
        \paren[\big]{R(x_{t+1}) - R(x_{t}) - D_R(x_{t+1}, x_t)}(t-1).
        \end{aligned}
    \end{equation*}  
\end{lemma}
\begin{proof}
    Let \(t \geq 1\). By the definition of the FTRL algorithm, we have
    \(x_t \in \argmin_{x \in \Xcal}(F_{t-1}(x) + (t-1) \Psi(x)+ \frac{1}{\eta_{t-1}}
    R(x))\). By the optimality conditions for convex programs, we have
    \begin{equation*}
        \subdiff[\paren*{F_{t-1} + (t-1) \Psi(x) + \tfrac{1}{\eta_{t-1}} R}](x_t) \cap (- N_{\Xcal}(x_t)) \neq \emptyset.
    \end{equation*}
    Since \(\subdiff[\paren{F_{t-1} + (t-1) \Psi(x) + \tfrac{1}{\eta_{t-1}} R}](x_t) =
    \subdiff[(F_{t-1} + (t-1) \Psi(x))](x_t) + \tfrac{1}{\eta_{t-1}} \nabla R(x_t)\), the
    above shows there is \(p_t \in N_{\Xcal}(x_t)\) such that
    \begin{equation*}
        - p_t - \frac{1}{\eta_{t-1}}\nabla R(x_t)
        \in \subdiff[(F_{t-1} + (t-1) \Psi(x))](x_t).
    \end{equation*}
    Using the subgradient inequality~\eqref{eq:subgradient_ineq} with
    the above subgradient yields,
    \begin{align*}
        &F_{t-1}(x_t) + (t-1)\Psi(x_t) - F_{t-1}(x_{t+1}) - (t-1)\Psi(x_{t+1})\\
        &\leq -\iprod{p_t}{x_t - x_{t+1}}
        - \tfrac{1}{\eta_{t-1}} \iprod{\nabla R(x_t)}{x_t - x_{t+1}},
        \\
        &\leq 
        - \tfrac{1}{\eta_{t-1}} \iprod{\nabla R(x_t)}{x_t - x_{t+1}}
        & \text{(by the definition of normal cone),} 
        \\
        &=\tfrac{1}{\eta_{t-1}}
        \paren[\big]{R(x_{t+1}) - R(x_{t}) - D_R(x_{t+1}, x_t)},
    \end{align*}
    where in the last equation we used that, by definition of the
    Bregman divergence, \(D_R(x_{t+1}, x_t) = R(x_{t+1}) - R(x_t) -
    \iprod{\nabla R(x_t)}{x_{t+1} - x_t}\) and, thus, \( - \iprod{\nabla
    R(x_t)}{x_t - x_{t+1}} = R(x_{t+1}) - R(x_t) - D_R(x_{t+1}, x_t)\).
\end{proof}

Now we are in position to prove Theorem~\ref{theorem8}.

\begin{proof}[{Proof of Theorem~\ref{theorem8}}]

    We proceed in a way extremely similar to the proof of
    Theorem~\ref{thm:ftrl_sublinear_regret}, but in place of the
    standard FTRL Lemma we use its composite version as
    in~\eqref{eq:composite_strong_ftrl_lemma}.

    For each \(t \geq 0\) let \(H_t\) be define das in the (composite)
    Strong FTRL Lemma so that $H_t = \sum_{i=1}^tf_i + t\Psi +
    \frac{1}{\eta_t}R$ and fix \(t \geq 0\). In this case we have
    \begin{equation*}
        H_t(x_t)-H_t(x_{t+1}) = F_t(x_t) - F_t(x_{t+1}) + t (\Psi(x_t) - \Psi(x_{t+1}))
        +\frac{1}{\eta_t}(R(x_t)-R(x_{t+1})).
    \end{equation*}
    Using \(F_t = F_{t-1} + f_t\) together with Lemma~\ref{lemma9} we
    have
    \begin{align*}
        &F_t(x_t) - F_t(x_{t+1}) + t(\Psi(x_t) - \Psi(x_{t+1}))
        \\
        &\leq \frac{1}{\eta_{t-1}}
        \paren[\big]{R(x_{t+1}) - R(x_{t}) - D_R(x_{t+1}, x_t)}
        + f_t(x_t) - f_t(x_{t+1}) + \Psi(x_t) - \Psi(x_{t+1}).
    \end{align*}
    Proceeding as in the proof of
    Theorem~\ref{thm:ftrl_sublinear_regret} (with the addition of a
    \(\Psi(x_t) - \Psi(x_{t+1})\) term) we have 
    \begin{equation*}
        H_t(x_t) - H_t(x_{t+1})
        \leq \frac{L^2\eta_{t-1}}{2}  +\paren[\Big]{\frac{1}{\eta_t} - 
        \frac{1}{\eta_{t-1}}}(R(x_t) - R(x_{t+1})) + \Psi(x_t) - \Psi(x_{t+1}).
    \end{equation*}
    When summing over \(t \in \{1, \dotsc, T\}\), the terms \(\Psi(x_t)
    - \Psi(x_{t+1})\) telescope so that, since \(x_1\)
    minimizes~\(\Psi\), we have
    \begin{equation*}
        \sum_{t = 1}^T (\Psi(x_t) - \Psi(x_{t+1})
        = \Psi(x_1) - \Psi(x_{T+1}) \leq 0.
    \end{equation*}
    Therefore, the remainder of the proof follows as in the proof of Theorem~\ref{thm:ftrl_sublinear_regret}.
\end{proof}

\subsection{Regularized Dual Averaging}

As previously discussed, applying OCO algorithms such as dual averaging
in an out-of-the-box fashion when the loss functions are composite case
does not exploit the structure of the extra-regularization given
by~\(\Psi\) and may have poor performance in practice. For example,
\citet{mcmahan2017survey} shows that applying DA in the composite case
with \(\Psi \coloneqq
\norm{\cdot}_1\) does not yield sparse solutions. \citet{xiao2010dual}
proposed the \emph{regularized dual averaging} (RDA) method to solve
this issue. The algorithm is identical to DA but it \emph{does not
linearize} the function \(\Psi\). Formally, the initial iterate $x_1$ is
in \(\argmin_{x \in \Xcal}(R(x)\) and is such that \(\Psi(x_1) = 0\),
that is, \(x_1\) minimizes \(\Psi\). For the following rounds, RDA
computes
\begin{equation}\label{RDA}
x_{t+1} \in \argmin_{x\in\mathcal{X}}\paren[\Big]{\sum_{i=1}^t\iprod{g_i}{x}+ t\Psi(x)+\frac{1}{\eta_t}R(x)} \qquad \forall t\ge1.
\end{equation}
With an argument analogous to the one made in Section~\ref{sec:da}, we
can write RDA as an instance of FTRL (with composite loss functions) and
obtain the following corollary of Theorem~\ref{theorem8}.

\begin{corollary}
    %
    Let \(\Psi \colon \Reals^n \to \Reals_+\) be a nonnegative convex
    function. Let \(\curly{x_t}_{t \geq 1}\) be defined  as
    in~\eqref{RDA} and assume \(\Psi(x_1) = 0\). Moreover, suppose $f_t$
    is $L$-Lipschitz continuous relative to $R$ for all~$t \geq 1$. Let
    \(z \in \Xcal\) and let $K \in \Reals$ be such that $ K \geq  R(z) -
    R(x_1)$. If $\eta_t\coloneqq\sqrt{2K}/(L\sqrt{t+1})$ for all \(t
    \geq 1\), then~\(\Regret_T^{\Psi}(z) \leq 2 L \sqrt{K(T+1)}\).

\end{corollary}

\section{Proofs for Section \ref{section4}}

In this section we give the missing proofs of Section \ref{section4}.
Throughout this section, let \(\curly{x_t}_{t \geq 1}\) and
$\curly{\hat{w}_t}_{t\ge1}$ \ be defined as in
Algorithm~\ref{algorithm:DS-OMD}, and define
\begin{equation*}
    w_t \coloneqq \nabla \Phi^*(\hat{w}_t), \qquad \forall t \geq 1.
\end{equation*}
First, let us state inequality (4.9) and Claim 4.2 (without substituting
exactly value of $\gamma_t$) from \citet{HuangHPF20a} at the beginning,
which will appear multiple times throughout this section, respectively
as:
\begin{claim}\label{claim10}
     If $\gamma_t=\eta_{t+1}/\eta_t\in(0,1]$ for each \(t \geq 1\), then
\begin{align*}
    &f_t(x_t)-f_t(z)\le\frac{1}{\eta_t}(D_\Phi(x_t,w_{t+1})-D_\Phi(z,w_{t+1})+D_\Phi(z,x_t)).
\end{align*}
\end{claim}
\begin{claim}\label{claim11}
If $\gamma_t\in(0,1]$ for all \(t\geq 1\), then, 
\begin{align*}
    &\frac{1}{\eta_t}(D_\Phi(x_t,w_{t+1})-D_\Phi(z,w_{t+1})+D_\Phi(z,x_t))\\
    &\le \frac{D_\Phi(x_t,w_{t+1})}{\eta_t}+\frac{1}{\eta_t}\bigg(\Big(\frac{1}{\gamma_t}-1\Big)D_\Phi(z,x_1)-\frac{1}{\gamma_t}D_\Phi(z,x_{t+1})+D_\Phi(z,x_t)\bigg).
\end{align*}
\end{claim}
\subsection{Sublinear Regret for Relative Lipschitz Functions}\label{App:dsomdsublinear}
In this subsection we prove sublinear regret for DS-OMD with relative
Lipschitz continuous cost functions. First we use Theorem 4.1 in
\citet{HuangHPF20a}. This theorem is analogous to the bound given in the
analysis of classic OMD given by \citet[Theorem 4.2]{bubeck2015convex}. 
\begin{theorem}[{\citet[Theorem~4.1]{HuangHPF20a}}]\label{regretboundDSOMD}
    If $\gamma_t\coloneqq\eta_{t+1}/\eta_t$ for each \(t \geq 1\), then
\begin{equation*}
\Regret_T(z) \le\sum_{t=1}^T\frac{D_\Phi(x_t,w_{t+1})}{\eta_t}+ \frac{D_\Phi(z,x_1)}{\eta_{T+1}}, \qquad \forall T > 0.
\end{equation*}
\end{theorem}
Now we are ready to use Theorem \ref{regretboundDSOMD} to prove Theorem \ref{DSOMDsub}.
\begin{proof}[{Proof of Theorem \ref{DSOMDsub}}] We first need to bound
    the terms $D_\Phi(x_t,w_{t+1})$ for each $t \geq 1$. Fix \(t \geq
    1\). By the three-point identity for Bregman divergences
    (see~\eqref{eq:breg_identity}),
\begin{align}
    D_\Phi(x_t,w_{t+1})&=-D_\Phi(w_{t+1},x_t)+\angles{\nabla\Phi(x_t)-\nabla\Phi(w_{t+1}),x_t-w_{t+1}}.\label{e4}
\end{align}
From the definition of the iterates in Algorithm~\ref{algorithm:DS-OMD},
we have $\eta_t g_t=\nabla\Phi(x_t)-\nabla\Phi(w_{t+1})$. Thus,
\begin{align}
    (\ref{e4})&=-D_\Phi(w_{t+1},x_t)+\eta_t\angles{g_t,x_t-w_{t+1}},\nonumber\\
    &\stackrel{\eqref{eq:rel_lip_v2}}{\leq} 
    -D_\Phi(w_{t+1},x_t)+\eta_tL\sqrt{2D_\Phi(w_{t+1},x_t)}
    \le \frac{\eta_t^2L^2}{2}, \label{bregmanbound}
\end{align}
where first inequality is from~\eqref{eq:rel_lip_v2} ( since \(f_t\) is
Lipschitz continuous relative to $\Phi$) and the second inequality comes
from the fact that \(\sqrt{\alpha \beta} \leq (\alpha + \beta)/2\) with
\(\alpha \coloneqq \eta_t^2 L^2\) and \(\beta \coloneqq
D_{\Phi}(w_{t+1}, x_t)\). Plugging the above in Theorem
\ref{regretboundDSOMD}, we get
\begin{align*}
    \Regret_T(z) 
    &\le \sum_{t=1}^T\frac{\eta_tL^2}{2}+\frac{D_\Phi(z,x_1)}{\eta_{T+1}}
    \le \sum_{t=1}^T\frac{\eta_tL^2}{2}+\frac{K}{\eta_{T+1}}.
\end{align*}
Setting $\eta_t\coloneqq\sqrt{K}/L\sqrt{t}$ for each $t\ge1$ and by
using Lemma \ref{A1} from Appendix~\ref{app:arithmetic_ineq}we have
\begin{align*}
    \Regret_T(z) &\le\frac{L^2}{2}\cdot\frac{\sqrt{K}2\sqrt{T}}{L}+K\frac{L\sqrt{T+1}}{\sqrt{K}} \le 2L\sqrt{K(T+1)}. \qedhere
\end{align*}
\end{proof}


\subsection{Proof for Theorem \ref{MD:logregret}}\label{A:MDlogregret}
 In this section we give a logarithmic regret bound for OMD the cost
functions are when relative Lipschitz continuous and relative strongly
convex, both relative to the mirror map. The first step in the proof is
the following claim given by modifying Claims \ref{claim10} and
\ref{claim11} and combining them together.
\begin{claim} \label{c1}
Assume that $\gamma_t=1$ for all $t\ge1$, then
\begin{align*}
f_t(x_t)-f_t(z)
&\le\frac{1}{\eta_t}\big(D_\Phi(x_t,w_{t+1})-D_\Phi(z,x_{t+1})+D_\Phi(z,x_t)\big)
-MD_\Phi(z,x_t).
\end{align*}
\end{claim}
\begin{proof}[Proof of Claim \ref{c1}] This proof largely follows the
structure of the proof of Claim \ref{claim10}. First, instead of using
subgradient inequality, we use the definition of relative strong
convexity and get
\begin{equation*}
f_t(x_t)-f_t(z)\le \angles{g_t,x_t-z}-MD_\Phi(z,x_t).
\end{equation*}
By proceeding as in the proof of Claim \ref{claim10} but adding the
extra term $-MD_\Phi(z,x_t)$ term we get
\begin{align*}
    f_t(x_t)-f_t(z)
    \le\frac{1}{\eta_t}\big(D_\Phi(x_t,w_{t+1})-D_\Phi(z,w_{t+1})+D_\Phi(z,x_t)\big)
    -MD_\Phi(z,x_t).
\end{align*}
Then we apply Claim~\ref{claim11} with $\gamma_t = 1$ to get the desired
inequality.
\end{proof}
\noindent
The next step in the proof of the logarithmic regret bound is to sum
Claim~\ref{c1} over $t$, yielding
\begin{align*}
    &\sum_{t=1}^T\big(f_t(x_t)-f_t(z)\big)\\
    &\le \sum_{t=1}^T\frac{D_\Phi(x_t,w_{t+1})}{\eta_t}+\sum_{t=2}^T\bigg(\bigg(\frac{1}{\eta_t}-\frac{1}{\eta_{t-1}}\bigg)D_\Phi(z,x_t)-MD_\Phi(z,x_t)\bigg)\\
&+\frac{1}{\eta_1}D_\Phi(z,x_1)-\frac{1}{\eta_T}D_\Phi(z,x_{T+1})-MD_\Phi(z,x_1), &\text{(by Claim \ref{c1})}\\
&\le \sum_{t=1}^T\frac{D_\Phi(x_t,w_{t+1})}{\eta_t}+\sum_{t=2}^T\bigg(\bigg(\frac{1}{\eta_t}-\frac{1}{\eta_{t-1}}\bigg)D_\Phi(z,x_t)-MD_\Phi(z,x_t)\bigg). &(\eta_1=1/M)
\end{align*}
Since $\eta_t=\frac{1}{Mt}$, we have
\begin{align*}
\sum_{t=2}^T\bigg(\bigg(\frac{1}{\eta_t}-\frac{1}{\eta_{t-1}}\bigg)D_\Phi(z,x_t)-MD_\Phi(z,x_t)\bigg)
=\sum_{i=2}^T\bigg(MD_\Phi(z,x_t)-MD_\Phi(z,x_t)\bigg)=0.
\end{align*}
We have already shown that $D_\Phi(x_t,w_{t+1})\le\frac{\eta_t^2L^2}{2}$ in (\ref{bregmanbound}), so
\begin{align*}
    \Regret_T(z)
    &\le \sum_{t=1}^T\frac{D_\Phi(x_t,w_{t+1})}{\eta_t}+\sum_{i=2}^T\bigg(\bigg(\frac{1}{\eta_t}-\frac{1}{\eta_{t-1}}\bigg)D_\Phi(z,x_t)-MD_\Phi(z,x_t)\bigg),\\
&\le \sum_{t=1}^T\frac{\eta_tL^2}{2}=\frac{L^2}{2M}\sum_{t=1}^T\frac{1}{t}\le\frac{L^2}{2M}(\log T+1).
\end{align*}
The last step comes from upper bound of the harmonic series.

\subsection{Sublinear Regret for DS-OMD with Extra Regularization}\label{A:DSOMD}
Following the notation from Appendix \ref{app:ftrl_composite}, we let
$\Psi \colon \Xcal \to \Reals_+$ denote the extra regularizer, a
nonnegative convex function. We also assume $\Psi$ is minimized at $x_1$
with value $0$ and use composite regret to measure the performance. The
only modification we need to make to Algorithm \ref{algorithm:DS-OMD} is
to change the projection step of the algorithm to
\begin{equation}\label{projection}
    x_{t+1}=\argmin_{x\in\mathbb{R}^n}\paren[\big]{D_\Phi\big(x,y_{t+1}\big)+\eta_{t+1}\Psi(x)}.
\end{equation}
Here we minimize over \(\Reals^n\) instead of over \(\Xcal\) since we
can  introduce the constraint of the points lying in \(\Xcal\) by adding
to \(\Psi\) the indicator function of \(\Xcal\). That is, by adding to
\(\Psi\) the function
\begin{equation*}
    \delta_{\Xcal}(x) \coloneqq \begin{cases}
        0 &\text{if}~x \in \Xcal,\\
        + \infty &\text{otherwise},
    \end{cases}
    \qquad \forall x \in \Reals^n.
\end{equation*}
In the remainder of this section we denote by
$\Pi^{\Phi}_{\eta_{t+1}\Psi}(y_{t+1})$ the point computed by the
right-hand side of~\eqref{projection}. If we pick this projection
coefficient $\alpha_t$ carefully, we can get $O(\sqrt{T})$ regret, as
specified by the next theorem.

\begin{theorem} \label{t16}
    Let \(\curly{x_t}_{t \geq 1}\) be given as in
    Algorithm~\ref{algorithm:DS-OMD} with composite updates and with
    parameters $\gamma_t\coloneqq\eta_{t+1}/\eta_t$  for each \(t \geq
    1\). Assume that \(\Psi(x_1) = 0\) and that $f_t$ is $L$-Lipschitz
    continuous relative to $\Phi$ for all~$t \geq 1$. Let \(z \in
    \Xcal\) and~$K \in \Reals$ be such that $ K \geq D_{\Phi} (z,x_1)$.
    Then,
\begin{equation*}
\Regret_T^{\Psi}(z)\le\sum_{t=1}^T\frac{\eta_tL^2}{2}+\frac{K}{\eta_{T+1}},\qquad\forall z\in\mathcal{X},\forall T > 0.
\end{equation*}
In particular, for $\eta_t\coloneqq\sqrt{K}/L\sqrt{t}$ for each $t\ge1$,
then $\Regret_T^{\Psi}(z)\le2L\sqrt{K(T+1)}$.
\end{theorem}
The analysis hinges on the following generalization
of~\citep[Lemma~4.1]{bubeck2015convex}, which can be thought as a ``pythagorean Theorem'' for Bregman
projections.
\begin{lemma} \label{l14}
  Let \(x \in \mathbb{R}^n\), \(y \in \Dcalcirc\), and set \(\bar{y}
  \coloneqq \Pi_{\alpha_t \Psi}^\Phi(y)\). If \(\bar{y} \in \Dcalcirc\),
  then
  \begin{equation*}
    D_\Phi(x, \bar{y}) + D_\Phi(\bar{y}, y)
    \leq D_\Phi(x,y) + \alpha_t(\Psi(x) - \Psi(\bar{y})). 
  \end{equation*}
\end{lemma}
\begin{proof}[Proof of Lemma \ref{l14}]
  By the optimality conditions of the projection, we have $\nabla
    \Phi(y) - \nabla \Phi(\bar{y}) \in \partial(\alpha_t
    \Psi)(\bar{y}).$ Using the three-point identity of Bregman
    divergences (see~\eqref{eq:breg_identity}) and the subgradient
    inequality, we get
  \begin{align*}
    &D_\Phi(x, \bar{y}) + D_\Phi(\bar{y}, y) - D_\Phi(x,y)
    = \angles{\nabla \Phi(y) - \nabla \Phi(\bar{y}), x - \bar{y}}\leq \alpha_t (\Psi(x) - \Psi(\bar{y})).
  \end{align*}
Rearranging yields the desired inequality.
\end{proof}
We are now ready to prove Theorem \ref{t16}.
\begin{proof}[Proof of Theorem \ref{t16}] To prove the theorem, we just
need to show that Theorem \ref{regretboundDSOMD} still holds (with
respect to the composite regret) in the algorithm with composite
projections. We modify Claims \ref{claim10} and \ref{claim11} to get the
following claim.
\begin{claim} \label{c2}
\begin{align*}
&f_t(x_t)-f_t(z)\\
&\le\frac{D_\Phi(x_t,w_{t+1})}{\eta_t}+\left(\frac{1}{\eta_{t+1}}-\frac{1}{\eta_t}\right)D_\Phi(z,x_1)+\frac{D_{\Phi}(z,x_t)}{\eta_t}-\frac{D_\Phi(z,x_{t+1})}{\eta_{t+1}}+(\Psi(z)-\Psi(x_{t+1})).
\end{align*}
\end{claim}
\begin{proof}[Proof of Claim \ref{c2}]
Claim \ref{claim10} gives us the following inequality:
\begin{equation*}
f_t(x_t)-f_t(z)\le\frac{1}{\eta_t}(D_\Phi(x_t,w_{t+1})-D_\Phi(z,w_{t+1})+D_\Phi(z,x_t)).
\end{equation*}
Then we just need to modify Claim \ref{claim11} to bound the right side of the above inequality. Using Lemma \ref{l14}, we have
\begin{align*}
D_\Phi(z,y_{t+1})-D_\Phi(x_{t+1},y_{t+1})
&\ge D_\Phi(z,x_{t+1})+\alpha_t(\Psi(x_{t+1})-\Psi(z)).
\end{align*}
Then we substitute the step $D_\Phi(z,y_{t+1})-D_\Phi(x_{t+1},y_{t+1})\ge D_\Phi(z,x_{t+1})$ in the original proof of Claim \ref{claim11} in \citet{HuangHPF20a} with the above inequality plus the extra regularization term and Claim \ref{c2} follows.
\end{proof}
Now the regret is bounded by
\begin{align*}
    &\Regret_T^{\Psi}(z)\\
    &=\sum_{t=1}^T\bigg(f_t(x_t)+\Psi(x_t)-f_t(z)-\Psi(z)\bigg),\\
    &=\sum_{t=1}^T\bigg(\Big(f_t(x_t)-f_t(z)\Big)+\Big(\Psi(x_t)-\Psi(z)\Big)\bigg),\\
    &\le\sum_{t=1}^T\frac{D_\Phi(x_t,w_{t+1})}{\eta_t}+\sup_{z\in\mathcal{X}}\frac{D_\Phi(z,x_1)}{\eta_{T+1}}+\sum_{t=1}^T(\Psi(x_t)-\Psi(x_{t+1})),\\
&=\sum_{t=1}^T\frac{D_\Phi(x_t,w_{t+1})}{\eta_t}+\sup_{z\in\mathcal{X}}\frac{D_\Phi(z,x_1)}{\eta_{T+1}}+\Psi(x_1)-\Psi(x_{T+1}),\\
&\le \sum_{t=1}^T\frac{D_\Phi(x_t,w_{t+1})}{\eta_t}+\sup_{z\in\mathcal{X}}\frac{D_\Phi(z,x_1)}{\eta_{T+1}}.
\end{align*}
The first inequality follows Claim \ref{c2} and the last step comes from the assumption that $x_1$ is the minimizer of $\Psi$. This shows Theorem \ref{regretboundDSOMD} holds as desired and then the proof of Theorem \ref{t16} follows as in Appendix \ref{App:dsomdsublinear}.
\end{proof} Similarly, by
setting all $f_t$ to a fixed function $f$ and taking average we get the
following corollary.
\begin{corollary}
Consider a convex function $f$ and let $x^*$ be a minimizer of $f$. Let $\Phi$ be a differentiable strictly convex mirror map such that $\mathcal{X}\subseteq\Dcalcirc$. Assume that $f$ is $L$-Lipschitz continuous to $\Phi$ and there exists non-negative $K$ such that $K\ge D_\Phi(x^*,x_1)$. Let $\{\eta_t\}_{t\ge1}$ be a sequence of step sizes. If we pick step size $\eta_t=\frac{1}{\sqrt{t}}$, $\alpha_t=\eta_{t+1}$ and stabilization coefficient $\gamma_t=\eta_{t+1}/\eta_t$, then we have convergence rate
\begin{equation*}
(f+\Psi)\parens{\frac{1}{T}\sum_{t=1}^Tx_t}-(f+\Psi)(x^*)\le\frac{2L\sqrt{2K}}{\sqrt{T}}.
\end{equation*}
\end{corollary}


\end{document}